\documentclass[twocolumn,english]{IEEEtran}
\usepackage{amsmath,graphicx}

\pdfoutput=1

\DeclareMathOperator*{\minimize}{minimize}

\usepackage[T1]{fontenc}
\usepackage{babel}
\usepackage{units}
\usepackage{amsthm}
\usepackage{amssymb,bm}
\usepackage{graphicx}
\usepackage{amsbsy}
\usepackage{algorithm}
\usepackage{algorithmic}
\usepackage[unicode=true,
bookmarks=true,bookmarksnumbered=true,bookmarksopen=true,bookmarksopenlevel=1,
 breaklinks=true,pdfborder={0 0 0},backref=true,colorlinks=false]
 {hyperref}

\newtheorem{problem}{Problem}
\newtheorem{lemma}{Lemma}

\title{Communication-Aware Energy Efficient Trajectory Planning with Limited Channel Knowledge}

\usepackage{amsfonts}
\usepackage{mathrsfs}
\usepackage{latexsym}
\usepackage{color}

\def\xx{\mathsf{x}}
\def\yy{\mathsf{y}}
\def\zz{\mathsf{z}}
\def\mcal{\mathcal}
\def\eu{\mathfrak}
\newcommand{\Mr}[1]{{\mathrm {#1}}}
\newcommand{\dif}[1]{\mathrm{d}{#1}}
\newcommand{\der}[2]{\frac{\dif{#1}}{\dif{#2}}}
\newcommand{\ders}[3]{\frac{\mathrm{d}^{{#3}}{#1}}{\mathrm{d}{#2}^{{#3}}}}
\newcommand{\derT}[2]{{\mathrm{d}{#1}}/{\mathrm{d}{#2}}}
\newcommand{\dersT}[3]{{\mathrm{d}^{{#3}}{#1}}/{\mathrm{d}{#2}^{{#3}}}}

\newcommand{\pT}{\derT{}{t}}

\newcommand{\psT}[1]{\dersT{}{t}{{#1}}}

\newcommand{\depT}[2]{\partial{#1}/\partial{#2}}
\newcommand{\lp}{\mathrm{s}}

\newcommand{\Expm}[1]{{\exp}\left({#1}\right)}

\newcommand{\Rss}[2]{\left\langle {#1}\, \left|\, {#2}\right. \right\rangle}
\newcommand{\im}[1]{\mathrm{ Im } \hspace{0.5ex}{#1}}
\newcommand{\Rset}{\Bbb{R}}

\newcommand{\Id}[1]{\Mr{I}_{#1}}

\newcommand{\rk}[1]{\mathrm{ rank } \hspace{0.5ex}{#1}}

\author{Daniel Bonilla Licea, M. Bonilla, Mounir Ghogho, Samson Lasaulce and Vineeth S. Varma %
\thanks{Daniel Bonilla Licea is with the International University of Rabat, Morocco, e-mail:{daniel.bonilla-licea@uir.ac.ma}.}%
\thanks{M. Bonilla is with the CINVESTAV-IPN, UMI 3575 CINVESTAV-CNRS, Mexico, e-mail:{mbonilla@cinvestav.mx}.}%
\thanks{Mounir Ghogho is with the University of Leeds, UK and the International University of Rabat, Morocco, e-mail: {m.ghogho@ieee.org}.%
}
\thanks{Samson Lasaulce is with the CNRS, L2S (CNRS-CentraleSupelec-Univ. Saclay), France, e-mail: {lasaulce@lss.supelec.fr}.%
}
\thanks{Vineeth S. Varma is with the CRAN, Universit\'e de Lorraine, France, e-mail: {vineethsvarma@gmail.com}.}
}

\begin{document}
\maketitle

\begin{abstract}
Wireless communications is nowadays an important aspect of robotics. There are many applications in which a robot must move to a certain goal point while transmitting information through a wireless channel which depends on the particular trajectory chosen by the robot to reach the goal point. In this context, we develop a method to generate optimum trajectories which allow the robot to reach the goal point using little mechanical energy while transmitting as much data as possible. This is done by optimizing the trajectory (path and velocity profile) so that the robot consumes less energy while also offering good wireless channel conditions. We consider a realistic wireless channel model as well as a realistic dynamic model for the mobile robot (considered here to be a drone). Simulations results illustrate the merits of the proposed method. 
\end{abstract}
%
\small

\normalsize

\section{Introduction}
\label{sec:intro}
Communication-aware motion planning is a relatively new research area that has been gaining interest within the communications and robotics communities. In this research area, from a communications engineering perspective, the position of the transceiver, located on the robot, is another parameter of the communications system that can be controlled. The general idea behind communication-aware motion planning techniques is to control the position of the robot (or autonomous vehicle) in order to improve certain communications metrics while moving to certain places required by another task.

\subsection{State of the art}

Now we present some of the problems that have been considered within this research area. In \cite{r2,r3,r4} the authors present a control technique for a surveillance robot to compensate for the multipath fading in the wireless channel. In those articles, the robot has to follow a predefined surveillance trajectory and a control technique that adapts the velocity profile of the robot, according to the wireless channel measurements, in order to compensate for the multipath fading. Related to this, in \cite{r5,r6} we designed a technique to compensate for the multipath fading by making the robot explore some points in the vicinity of its initial position whose locations are optimized. 

The problem of designing communications-aware trajectories has attracted a lot of attention. For example, in \cite{r11} the authors derive a communication-aware trajectory planner based on harmonic potential fields which allow a robot to go from a starting point to a goal point in a finite time and in a cluttered environment while maintaining a reliable channel quality during the whole trajectory. Nevertheless, they assume full knowledge of the signal strength map, i.e., full knowledge of the wireless channel spatial variations. Another case of a communications-aware trajectory problem is found in \cite{r12,r13} where the authors design adaptive communication-aware trajectories for two different problems: the first consists of a robot that gathers channel measurements to estimate wireless channel parameters while the second problem involves a robot that must track a target while maintaining communication with an access point. In \cite{r14,r15} the authors extend their work to consider the case in which a team of robots communicating with a base station surveys an area looking for static targets. They design a trajectory planner for the robots that balances the sensing aspects of the task with the communications requirements. In \cite{r16} the authors consider the problem of a mobile robot that departs from a starting point and must reach, in a finite time, a target point by following a predetermined trajectory; in addition, the robot must transmit a finite amount of data to a base station. The authors devise a strategy to modify the velocity profile of the robot and the transmission rate in order to minimize the total amount of energy (due to both motion and communications). 

In \cite{r19}, the authors lift the restriction of the predetermined path and develop an optimum trajectory so that the drone departs from a starting point and reaches a goal point in finite time while transmitting a predefined number of bits and minimizing total energy consumption. In \cite{r17} the authors consider the case in which a mobile robot must visit a number of points of interest to gather information and then send it to a base station. They optimized the visiting order for the points of interest and the velocity profile of the robot in order to minimize the total energy consumption. They use a linear path between points of interest at all times. Also, some multi-robot scenarios have been considered in other recent works. For example, in \cite{r7,r8,r9,r10}, the authors consider the problem of designing trajectories for a robotic team in which the leader must attain a predetermined final position in finite time while maintaining a certain end-to-end transmission rate at all times; the environment (including the location of the obstacles) is assumed known. The cost function of the problem is a convex function of the desired final configuration parameters, and the environment (including the location of the obstacles) is assumed known prior to the execution of the trajectory.

In this article, we consider the problem of designing a predetermined trajectory for an autonomous robot which takes into account both energy consumption and communications aspects. More specifically, we develop a method to optimize a drone's trajectory to depart from a starting point $\mathbf{s}$ and reach a goal point $\mathbf{g}$ (determined by the user) in a finite time $t_f$ while communicating with an access point (AP) through a wireless communications channel experiencing large-scale fading (also called shadowing). The optimization is done taking into account simultaneously a communications related term (which can take various forms depending on the application) and the energy consumed by the robot by following the trajectory.

\subsection{Contribution}

Existing works addressing similar problems consider an arbitrary predetermined path and only optimize the velocity profile. To the authors' knowledge, \cite{r19} is the closest work to the problem considered in this article. Nevertheless, our work differs in two aspects: (i) in \cite{r19}, the problem considered is to transmit a predefined number of bits while minimizing the energy; in our work, the goal is to use as little energy as possible while optimizing a general communications metric (ii) in \cite{r19}, the authors consider a double integrator model for the robot while in our work we consider a more realistic dynamical model for the robot. Another difference of our work with respect to other works dealing with similar problems is that we consider both a realistic communications channel and a realistic robot model (which includes the dynamics constraints). 

In this work, we focus on designing predetermined trajectories (as opposed to adaptive trajectories), i.e. the trajectory is fully optimized before the drone starts following it. Further, the optimization is performed by assuming knowledge of only the first order statistics of the channel, which is more realistic in practice, while other works assume more knowledge of the channel, e.g. \cite{r19} or \cite{r11}. The trajectory optimization method presented in this article is also able to take into account obstacles. We will interchangeably use the terms robot, drone and quadrotor in this article.

\subsection{Organization}
The rest of this paper is organized as follows. Section \ref{sec:Model} presents the dynamical model of the drone as well as that of the communications system. The problem considered here is formally stated in Section \ref{sec:ProblemStatement}. A method to solve this problem is developed in Section \ref{sec:Solution}. Then in section \ref{sec:Obstacles} we show how to modify our method to take into account obstacles. Simulation results are presented in Section \ref{sec:Simulations} and finally conclusions are drawn in Section \ref{sec:Conclusions}.

\subsection{Notation}
$x^\mathrm{T}$ stands for the transpose of the matrix (or vector) $x$. Given a pair of mappings $(A, B)$, such that $A: \Rset^{n} \to \Rset^{n}$ and $B:$ $\Rset^{m}$ $\to$ $\Rset^{n}$, $\Rss{A}{\im{B}}$ stands for the reachability subspace and
$\mcal{C}_{(M,\,S)}$ stands for the controllability matrix, which are defined as
$\Rss{A}{\im{B}}$  $=$ ${\im{B}} + \sum_{i=1}^{n-1}A^{i}\,{\im{B}}$ and
$\mcal{C}_{(A,\,B)}$  = $\left[\begin{array}{cccc}
{B} & {A}\,{B} & \cdots & {A}^{(n-1)}\,{B}
\end{array}\right]$. Further, $\Sigma(A, B, C)$ stands for the state space system $\derT{x}{t}=Ax+Bu$ and $y=Cx$ where $x$ is the state vector, $u$ is the control signal and $y$ is the output of the system. $\mathrm{c}(\vartheta)$ and $\mathrm{s}(\vartheta)$ stand for cosine and sine of a given angle $\vartheta$, respectively. The ceiling and the floor functions are denoted as $\lceil\cdot\rceil$ and $\lfloor\cdot\rfloor$ respectively. $\hat{x}$ stands for the estimate of the variable $x$; $\mathbb{E}[\cdot]$ stands for the statistical expectation operator.

\section{System Model}
\label{sec:Model}
\subsection{Mobile robot dynamical model}
\label{sec:Model:Robot}
In this article we consider an autonomous quadrotor whose dynamical model is given by (see \cite{r1} for analytical details):
\begin{eqnarray}
\label{eq:QR_model}
\left[\begin{array}{c}
\ders{\xx}{t}{2}
\\ \noalign{\smallskip}
\ders{\yy}{t}{2}
\\ \noalign{\smallskip}
\ders{\zz}{t}{2}
\end{array}\right] &=&
\left[\begin{array}{c}
\mathrm{c}(\phi)\mathrm{s}(\theta)\mathrm{c}(\psi) +
\mathrm{s}(\phi)\mathrm{s}(\psi)
\\ \noalign{\smallskip}
\mathrm{c}(\phi)\mathrm{s}(\theta)\mathrm{s}(\psi) -
\mathrm{s}(\phi)\mathrm{c}(\psi)
\\ \noalign{\smallskip}
\mathrm{c}(\phi)\mathrm{c}(\theta)
\end{array}\right]\frac{{u}_{\zz}}{m} -
\left[\begin{array}{c}
0 \\ 0 \\ g
\end{array}\right]
\nonumber\\
\end{eqnarray}
\begin{eqnarray}
\label{eq:QR_model2}
\left[\begin{array}{c}
\ders{\phi}{t}{2}
\\ \noalign{\smallskip}
\ders{\theta}{t}{2}
\\ \noalign{\smallskip}
\ders{\psi}{t}{2}
\end{array}\right] &=&
\left[\begin{array}{c}
\left(\frac{I_\yy-I_\zz}{I_\xx}\right)\,\der{\theta}{t}\,\der{\psi}{t} -
\frac{J}{I_\xx}\,\der{\theta}{t}\,{q}_{w}
\\ \noalign{\smallskip}
\left(\frac{I_\zz-I_\xx}{I_\yy}\right)\,\der{\phi}{t}\,\der{\psi}{t}+
\frac{J}{I_\yy}\,\der{\phi}{t}\,{q}_{w}
\\ \noalign{\smallskip}
\left(\frac{I_\xx-I_\yy}{I_\zz}\right)\,\der{\phi}{t}\,\der{\theta}{t}
\end{array}\right] +
\left[\begin{array}{c}
\frac{\ell\,{u}_{\yy}}{I_\xx}
\\ \noalign{\smallskip}
\frac{\ell\,{u}_{\xx}}{I_\yy}
\\ \noalign{\smallskip}
\frac{\,{u}_{\psi}}{I_\zz}
\end{array}\right]
\nonumber\\
\end{eqnarray}
\begin{eqnarray}
\label{eq:Inp_vec}
\notag
\left[\begin{array}{c}
{u}_{\xx} \\ {u}_{\yy} \\ {u}_{\zz} \\ {u}_{\psi}
\end{array}\right] &=&
\left[\begin{array}{cccc}
-\kappa_b&0&\kappa_b&0\\
0&\kappa_b&0&-\kappa_b\\
\kappa_b&\kappa_b&\kappa_b&\kappa_b\\
\kappa_\tau&-\kappa_\tau&\kappa_\tau&-\kappa_\tau
\end{array}\right]
\left[\begin{array}{c}
{\omega}_{1}^{2} \\ {\omega}_{2}^{2} \\ {\omega}_{3}^{2} \\ {\omega}_{4}^{2}
\end{array}\right]
\\
\end{eqnarray}
\begin{eqnarray}
\label{eq:qw_vec}
{q}_{w}&=&\omega_1 - \omega_2 + \omega_3 - \omega_4
\end{eqnarray}

\noindent
where ${u}_{\xx}(t)$ ${u}_{\yy}(t)$ ${u}_{\zz}(t)$ and ${u}_{\psi}(t)$ denote the control signals for the drone;  $\omega_j(t)$ is the angular velocity of the $j$th motor; $m$ is the total mass of the drone, $g$ is the gravitational constant; $\ell$ is the distance from the center of the quadrotor to each motor;  $I_\xx$, $I_\yy$ and $I_\zz$ are the rotational inertia components; $J$ is the total inertia of the motors; and $\kappa_b$ and $\kappa_\tau$ are the thrust and aerodynamic drag factors of the propellers\footnote{Note that the matrix relating the vector inputs 
$\left[\begin{array}{cccc}
{u}_{\zz} & {u}_{\yy} & {u}_{\xx} & {u}_{\psi}
\end{array}\right]^{T}
$ with the square angular velocities vector
$\left[\begin{array}{cccc}
{\omega}_{1}^{2}(t) & {\omega}_{2}^{2}(t) & {\omega}_{3}^{2}(t) & {\omega}_{4}^{2}(t) \end{array}\right]^{\mathrm{T}}$ is not singular.}.  

Equation (\ref{eq:QR_model2}) describes the drone's Euler angles (roll $(\phi)$, pitch $(\theta)$ and yaw $(\psi)$) measured with respect to the axes $o_{B}{\xx}_{B}$, $o_{B}{\yy}_{B}$ and $o_{B}{\zz}_{B}$, with $(o_{B}{\xx}_{B}{\yy}_{B}{\zz}_{B})$ being the body axis system whose origin ${o_{B}}$ is given by the geometric centre of the quadrotor. 

The motion of the quadrotor is described with respect to a fixed orthogonal axis set ${(o{\xx}{\yy}{\zz})}$, where ${o{\zz}}$ points vertically up along the gravity vector ${\left[\begin{array}{ccc} 0&0&-g \end{array}\right]^{T}}$ (earth axes). The origin ${o}$ is located at a desired height ${\bar{\zz}}$ with respect to the ground level. The coordinates  $\xx$, $\yy$ and $\zz$ in (\ref{eq:QR_model}) refer to the position of the centre of gravity of the quadrotor in the space where $\zz$ is its altitude \cite{Cook}.

The time dependence of the variables in (\ref{eq:QR_model}) and (\ref{eq:QR_model2}) is not explicitly shown in order to lighten the notation. Note also that, due to the symmetry of the quadrotor, we have $I_\xx=I_\yy=I$.

The dynamical model of the quadrotor described by equations (\ref{eq:QR_model})-(\ref{eq:Inp_vec}) is nonlinear which makes it difficult to control. So, in the Appendix \ref{Appendix:A} we introduce a new linearization method which transforms the nonlinear quadrotor model (\ref{eq:QR_model})-(\ref{eq:Inp_vec}) into the following linear model which we will use in the rest of the article:
\begin{equation}
\label{eq:Rep_edo_lin-pert2}
\begin{array}{l}
\derT{{\zeta}_{\zz}}{t} = {A}_{\zz}\,{\zeta}_{\zz},
\ \
\derT{{\zeta}_{\psi}}{t} = {A}_{\psi}\,{\zeta}_{\psi},
\\ \noalign{\smallskip}
\derT{{\zeta}_{\xx}}{t} = {A}_{\xx}\,{\zeta}_{\xx} +
{B}_{\xx}\bar{u}_{\xx},
\ \
{p}_{\xx} = {C}_{\xx}\,{\zeta}_{\xx},
\\ \noalign{\smallskip}
\derT{{\zeta}_{\yy}}{t} = {A}_{\yy}\,{\zeta}_{\yy} +
{B}_{\yy}\bar{u}_{\yy},
\ \
{p}_{\yy} = {C}_{\yy}\,{\zeta}_{\yy}.
\end{array}
\end{equation}
where ${\zeta}_{\xx}$, ${\zeta}_{\yy}$, and ${\zeta}_{\psi}$ are state variables defined in Appendix \ref{Appendix:A}, and $\bar{\mathbf{u}}(t) \triangleq[\bar{u}_{\xx}(t)\ \ \bar{u}_{\yy}(t)]^\mathrm{T}$ is the new control signal after the linearization procedure. As mentioned in Appendix \ref{Appendix:A} the model (\ref{eq:Rep_edo_lin-pert2}) allows to control the position of the quadrotor in the $\xx-\yy$ plane of the fixed axis set ($o\xx\yy\zz$) which we denote $\mathbf{p}(t) \triangleq[{p}_{\xx}(t)\ \ {p}_{\yy}(t)]^\mathrm{T}$. In addition to the linearization procedure, the altitude and the yaw angle are regulated in order to make them remain constant.

Concerning the energy consumed by the robot from time $t_0$ to $t_1$,  we will assume for simplicity the following model:
\begin{equation}
E(t_0,t_1)=\displaystyle\int_{t_0}^{t_1}\|\bar{\mathbf{u}}(t)\|_2^2\mathrm{d}t.
\end{equation}

It is worth pointing out that while in this paper we illustrate the proposed trajectory design method using a quadrotor, the same methodology can be applied to other types of robots. Note also that we consider a full 3D dynamical nonlinear model of the quadrotor which we then linearize; we also restrict the movement of the robot to lie on a horizontal plane. By doing this we are able to still take into account the full dynamics\footnote{This makes the problem more realistic and allows to design trajectories that indeed can be followed by the robot under consideration.} of the quadrotor during the trajectory design.

\subsection{Communication system model}
As mentioned before, we assume that the wireless channel between the quadrotor and the AP experiences large-scale fading (a.k.a. shadowing) \cite{r20} and so the signal received by the AP from the quadrotor is:
\begin{equation}
\label{eq.1.2.1}
r_{\mathrm{AP}}(t,\mathbf{p}(t))=\left(\frac{h(\mathbf{p}(t))}{\|\mathbf{p}(t)\|_2^{\alpha/2}}\right) x_{\mathrm{Tx}}(t)+z_{\mathrm{AP}}(t),
\end{equation}
where $\alpha$ is the power path loss coefficient; $x_{\mathrm{Tx}}(t)$ is the signal transmitted by the drone with average transmission power $\mathbb{E}[|x_{\mathrm{Tx}}(t)|^2]=P$; and $z_{\mathrm{AP}}(t)$ is a zero mean additive white Gaussian noise (AWGN) with variance $\sigma^2$ at the AP's receiver. The shadowing is represented by $h({\mathbf{p}(t)})$ which is assumed to be lognormally distributed \cite{r18}. From (\ref{eq.1.2.1}), the signal-to-noise ratio (SNR) at the AP is:
\begin{equation}
\label{eq.1.2.4}
\Gamma(\mathbf{p}(t))=\frac{h^2(\mathbf{p}(t)) P}{\|\mathbf{p}(t)\|_2^{\alpha}\sigma^2}
\end{equation}

Now, the drone communicates with the AP using time duplexing \cite{b10} with period $T$. During the reception phase, the AP transmits a pure tone so that the drone estimates the signal-to-noise ratio (SNR). Then, during the transmission phase the drone transmits data to the AP at a certain data rate $R\left({\Gamma}(\mathbf{p}(kT))\right)$ which depends on the SNR estimated during the reception phase of the current duplexing period:
\begin{equation}
\label{eq.1.2.5}
R\left(\hat{\Gamma}(\mathbf{p}(kT))\right)=R_j, \ \ \ \forall\ \hat{\Gamma}(\mathbf{p}(kT))\in[\gamma_{j},\gamma_{j+1}),
\end{equation}
where $j=0,1,\cdots,J$, with $J$ being the number of different bit rates (different from zero) supported by the drone and the AP; $R_j<R_{j+1}$, $\gamma_j<\gamma_{j+1}$, $R_0=0$, $\gamma_0=0$; and $\gamma_1$ is a value that must be above the sensitivity of the AP's receiver. Note that the number of bits transmitted during the transmission phase of each duplexing period is $ T_{tx} R\left(\hat{\Gamma}(\mathbf{p}(kT))\right)$ where $T_{tx}<T$ is the duration of the transmission phase. Finally we assume that the energy spent by the robot due to communication is negligible compared to energy due to motion.

\section{Problem Statement}
\label{sec:ProblemStatement}
We want to design a predetermined trajectory for the drone to go from a starting point $\mathbf{s}$ to a predefined goal point $\mathbf{g}$ in a finite time $t_f$ while communicating with an AP. The optimization will be done taking into account simultaneously two terms: i) the energy consumed by the quadrotor in motion and ii) a communications related term.

To do this we devise the cost function as the convex combination of two terms: (i) the energy consumed by the quadrotor when following the trajectory; (ii) a communications related metric represented by a function of the expected number of bits transmitted during the trajectory tracking. Hence the cost function: 
\begin{equation}
\label{eq:2.0.1}
\mathcal{J}(\bar{\mathbf{u}},\lambda,t_f)=\frac{\lambda}{E_0} \displaystyle\int_0^{t_f}\left\|\bar{\mathbf{u}}(t)\right\|_2^2\mathrm{d}t
+(1-\lambda)w\left(\mathcal{D}_{\lfloor\frac{t_f}{T}\rfloor}\right)
\end{equation}
where $\mathcal{D}_{\lfloor\frac{t_f}{T}\rfloor}=T_{tx}\displaystyle\sum_{k=0}^{\lfloor\frac{t_f}{T}\rfloor}\mathbb{E}\left[R\left(\Gamma(\mathbf{p}(kT))\right)\right]$; $\lambda$ determines the relative importance of minimizing the mechanical energy criterion with respect to the communications criterion; $E_0$ is a normalization factor consisting of the energy spent by the robot for moving from $\mathbf{s}$ to a goal point $\mathbf{g}$ within a duration $t_f$ using a minimum energy control. As a consequence of this normalization, the range of variation of the first term is $[1,+\infty)$. The function $w(\cdot)$ is a general non-linear function which can take different forms depending on the communications criteria considered which will depend on the particular communications aware trajectory planning problem considered. The argument of $w(\cdot)$ is the number of bits that can be transmitted during the trajectory and the expectation\footnote{Since we are designing a predefined trajectory we do not have access to wireless channel measurements. In addition, the wireless channel is a random process, thus the need of taking the expected value.} of this argument is taken with respect to the shadowing term (i.e., with respect to $h(\mathbf{p}(kT))$):
\begin{equation}
\label{eq:2.0.2}
\mathbb{E}\left[R\left(\Gamma(\mathbf{p}(kT))\right)\right]\triangleq\int_0^{+\infty}R\left(\frac{x^2 P}{\|\mathbf{p}(kT)\|_2^{\alpha}\sigma^2}\right)f_s(x)\mathrm{d}x,
\end{equation} 
where $f_s(x)$ is the probability density function of the shadowing which corresponds to a lognormal distribution as mentioned before.

As mentioned above, the function $w(\cdot)$ can take many forms depending on the particular problem considered. To illustrate this, we briefly present two different communications aware trajectory planning optimization problems and the form that $w(\cdot)$ could take in each case\footnote{The only numerical requirement is that the range of variation of $w(\cdot)$ must be similar to that of the first term in (\ref{eq:2.0.1}) for numerical convenience.}. In both problems the quadrotor has to go from a starting point $\mathbf{s}$ to a predefined goal point $\mathbf{g}$ in a finite time $t_f$ while communicating with an AP but solving different problems:
\begin{enumerate}
\item {\it Minimum energy-maximum data}: in this problem we want to devise a trajectory in which the quadrotor uses little energy in motion while transmitting  as much data as possible to the AP. In this case the communications function $w$ can take the following form:
\begin{equation}
\label{eq:2.0.3}
w\left(\mathcal{D}_{\lfloor\frac{t_f}{T}\rfloor}\right)=W_0\left(\mathcal{D}_{\lfloor\frac{t_f}{T}\rfloor}\right)^{-1}
\end{equation}
where $W_0=T_{tx}R_J(\lfloor t_f/T\rfloor+1)$ is a normalization term. The function $w$ consists of the supremum of the number of bits that can be transmitted during the time $t_f$ divided over the expected number of bits transmitted. Note that (\ref{eq:2.0.3}) is a decreasing function of the number of bits transmitted during the trajectory.

\item {\it Minimum energy- fixed transmission quota}: in this problem we want to devise a trajectory in which the quadrotor uses little energy in motion while transmitting an average number of $N_0$ bits to the AP. In this case the communications function $w$ can take the following form:
\begin{equation}
\label{eq:2.0.4}
w\left(\mathcal{D}_{\lfloor\frac{t_f}{T}\rfloor}\right)=\exp\left(\eta\left(N_0-\mathcal{D}_{\lfloor\frac{t_f}{T}\rfloor}\right)\right)
\end{equation}
where $\eta>0$ is a constant with a very large value. The function (\ref{eq:2.0.4}) is a decreasing with respect of the number of bits transmitted during the trajectory. If $\eta$ is large enough then when $\displaystyle T_{tx}\sum_{k=0}^{\lfloor t_f/T\rfloor}\mathbb{E}\left[R\left(\Gamma(\mathbf{p}(kT))\right)\right]>N_0$ the communications function $w$ tends to zero but when the opposite happens then the communications function tends to infinity. In this case (\ref{eq:2.0.4}) operates as as penalization term. 

\end{enumerate}
From (\ref{eq:2.0.3})-(\ref{eq:2.0.4}) we can see that indeed the cost function (\ref{eq:2.0.1}) can be used to describe different communications aware trajectory planning optimization problems.

Now, optimizing (\ref{eq:2.0.1}) with respect to the control law $\bar{\mathbf{u}}$, under the constraints which describe the dynamical model of the drone and the requirements for its trajectory mentioned at the beginning of this section, is a complex problem. To the authors' knowledge such a problem cannot be solved analytically and so it must be solved numerically. Further, the fact that the expected value in (\ref{eq:2.0.1}) is a continuous nonlinear function of the drone's position $\mathbf{p}(kT)$ increases the computational complexity of the problem. All of these reasons motivate us to search for a sub-optimal, yet tractable method of solving this problem for any form of the communications function $w$ as long as it remains a function of the same argument. 

\section{Solution}
\label{sec:Solution}
The existing works that have addressed a somewhat similar problem either consider a fixed path and only optimize the velocity profile (e.g., \cite{r16}), assume a simplistic channel model or full channel knowledge, (e.g. \cite{r11},  \cite{r20}) or assume a simplified robot model \cite{r19}. In our work, we jointly optimize the path and velocity profiles of the robot, while considering realistic channel models and a comprehensive dynamical model for the robot. The optimization is done taking into account the mechanical energy spent in motion and a general communications related term which is a function of the expected number of bits transmitted. We develop a new method to solve this class of communications aware trajectory optimization problem. Note that the presence of obstacles will not be considered until section \ref{sec:Obstacles}.

\subsection{Rate quantization}
As mentioned before, optimizing (\ref{eq:2.0.1}) is an extremely complicated problem and it is not possible to do it analytically. In this method we will approximate some parts of the original optimization problem in order to be able to use analytical results in its resolution.

Let us start by looking at the typical shape of $\mathbb{E}[R\left({\Gamma}(\mathbf{p}(kT))\right)]$ in Fig. \ref{Figure0}. As we can observe, in general, it tends to present large regions which are significantly flat as long as the variance of the shadowing is not too large. Hence it can be well approximated by a quantized version with a reduced number of quantization levels. This approximation will allow us (as will be observed later in this section) to decouple the control aspects from the communications aspects in the optimization problem.
\begin{figure}[h]
\centerline{\includegraphics[clip, trim ={2mm 1mm 2mm 1mm}, height=4cm,width=8cm]{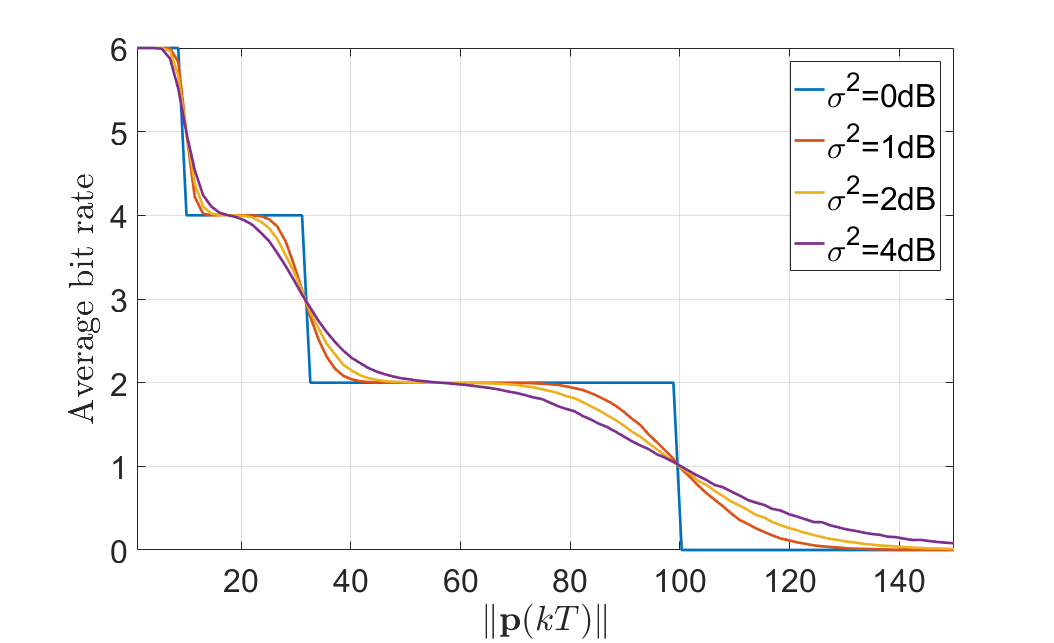}}
\vspace{-2mm}
\caption{Average bit rate $\mathbb{E}[R\left({\Gamma}(\mathbf{p}(kT))\right)]$ for $R_0=0$, $R_1=2$, $R_2=4$, $R_3=6$, $\eta_1=0.001$, $\eta_2=0.01$, $\eta_3=0.1$,  $\alpha=2$ and different shadowing variances.}
\label{Figure0}
\end{figure}

Then, the first step in the resolution method is to replace $\mathbb{E}[R\left({\Gamma}(\mathbf{p}(kT))\right)]$ by its quantized version which we will refer to as $f_Q(\mathbf{p}(kT))$ with $Q$ a parameter design that determine the number of quantization levels. Larger number of quantization levels $Q$ will result in a better approximation, but also a higher computational complexity. The function $f_Q(\mathbf{p}(kT))$ can be derived from $\mathbb{E}[R\left({\Gamma}(\mathbf{p}(kT))\right)]$ using a uniform quantization or an optimized quantization method as the one observed in Appendix \ref{Appendix:C}.

Once we replace $\mathbb{E}[R\left({\Gamma}(\mathbf{p}(kT))\right)]$ in (\ref{eq:2.0.1}) with $f_Q(\mathbf{p}(kT))$ the cost function becomes:
\begin{eqnarray}
\label{eq:3.0.6}
\mathcal{J}'(\bar{\mathbf{u}},\lambda,t_f)&=&\displaystyle\frac{\lambda}{E_0} \displaystyle\int_0^{t_f}\left\|\bar{\mathbf{u}}(t)\right\|_2^2\mathrm{d}t\nonumber\\
&+& (1-\lambda)w\left(T_{tx}\displaystyle\sum_{k=0}^{\lfloor\frac{t_f}{T}\rfloor}f_Q(\mathbf{p}(kT))\right).
\end{eqnarray}

Taking all of the above into account, the problem of optimizing a trajectory that makes the quadrotor go from a starting point $\mathbf{s}$ to a goal point $\mathbf{g}$ in a time $t_f$ while using little energy in motion and minimizing a communications related metric can be stated mathematically as:
\begin{equation}
\label{eq:3.0.7}
    \begin{array}{l}
\displaystyle \minimize_{\bar{\mathbf{u}}}\ \ \ \ \ \ \mathcal{J}'(\bar{\mathbf{u}},\lambda,t_f)\\
{\rm s.t.}\\
\derT{{\zeta}_{\xx}}{t} = {A}_{\xx}\,{\zeta}_{\xx} +
{B}_{\xx}\bar{u}_{\xx},\ \
{p}_{\xx} = {C}_{\xx}\,{\zeta}_{\xx},\\
\derT{{\zeta}_{\yy}}{t} = {A}_{\yy}\,{\zeta}_{\yy} +
{B}_{\yy}\bar{u}_{\yy},\ \
{p}_{\yy} = {C}_{\yy}\,{\zeta}_{\yy},\\
{{\zeta}_\xx(0)} =
\left[\begin{array}{cccc}
s_{\xx} & 0 & 0 & 0
\end{array}\right]^{\mathrm{T}}
,\ 
{{\zeta}_\yy(0)} =
\left[\begin{array}{cccc}
s_{\yy} & 0 & 0 & 0
\end{array}\right]^{\mathrm{T}}
,\\
{{\zeta}_\xx(t_f)} =
\left[\begin{array}{cccc}
\mathrm{g}_{\xx} & 0 & 0 & 0
\end{array}\right]^{\mathrm{T}}
,\ 
{{\zeta}_\yy(t_f)} =
\left[\begin{array}{cccc}
\mathrm{g}_{\yy} & 0 & 0 & 0
\end{array}\right]^{\mathrm{T}}
,\\
\end{array}
\end{equation}
where ${s}_\xx$, ${s}_{\yy}$,  $\mathrm{g}_\xx$ and $\mathrm{g}_{\yy}$ are the $x$ and $y$ components of $\mathbf{s}$ and $\mathbf{g}$ respectively. We remind the reader that $\bar{\mathbf{u}}(t) \triangleq[\bar{u}_{\xx}\ \ \bar{u}_{\yy}]^\mathrm{T}$ and also that $\mathcal{J}'(\bar{\mathbf{u}},\lambda,t_f)$ depends on $\mathbf{p}(t)=[{p}_{\xx}(t)\ \ {p}_{\yy}(t)]^\mathrm{T}$, the position of the drone (see (\ref{eq:3.0.6})) which is related to the state vectors ${\zeta}_{\xx}$ and ${\zeta}_{\yy}$ as indicated in subsection \ref{sec:Model:Robot}. The constraints in (\ref{eq:3.0.7}) describe the dynamical model of the drone and the fact that it must depart from the starting point $\mathbf{s}$, reach $\mathbf{g}$ at time instant $t_f$ and are still at both ends of the trajectory. In theory we could solve (\ref{eq:3.0.7}) using dynamic programming but it would require an extremely large amount of calculations making this approach intractable. So, in the next section we derive another method which produces a suboptimal solution by taking advantage of the fact that $f_Q(\mathbf{p}(kT))$ is a discrete function of quadrotor's position $\mathbf{p}(t)$.

The method proposed to solve the optimization problem (\ref{eq:3.0.7}) is divided in two phases. The first phase, presented in \ref{sec:Solution:decomposition},  takes an top-down approach to decompose (\ref{eq:3.0.7}) into a sequence of control problems that are simpler to handle, using standard results in optimal control theory, and in which the communication aspect does not explicitly appear. The second phase, presented in \ref{sec:Solution:integration},  takes a bottom-up approach to solve and combine the solutions of the smaller problems derived in the previous phase. After combining all the solutions, we obtain a (suboptimal\footnote{The solution derived is suboptimal due to the approximation made in \ref{sec:Solution:decomposition} which makes the problem more tractable.}) solution to (\ref{eq:3.0.7}). The reasoning behind this approach is that it is easier to solve multiple simple problems than a single complex problem. In the next subsection we describe the top-down decomposition.
 
\subsection{Top-down Decomposition}
\label{sec:Solution:decomposition}

The first step in the decomposition of (\ref{eq:3.0.7}) is to divide it in $Q-1$ different problems. To do this, we start by defining the regions $\mathcal{A}_j$ as:
\begin{equation}
\label{eq:3.0.8}
\mathcal{A}_j\triangleq\{\mathbf{p}\ |\ f_Q({R}\left(\Gamma(\mathbf{p}(kT))\right))=R_j^Q\},\  \ j=1,2,\cdots,Q
\end{equation}
where $R_j^Q>R_{j-1}^Q$. Due to the radial symmetry of the wireless channel model the region $\mathcal{A}_Q$ is circular while regions $\{\mathcal{A}_j\}_{j=2}^{Q-1}$ have ring shapes. We will refer to the circle dividing regions $\mathcal{A}_j$ and $\mathcal{A}_{j+1}$ as the $j$th border.  

Each of the $Q-1$ optimization problems\footnote{One problem per value of $j$.} mentioned above will have the same form as (\ref{eq:3.0.7}) but with the additional constraint that the drone must pass through the regions $\{\mathcal{A}_k\}_{k=1}^{j}$. To add such constraints, we first define $\mathbf{i}_k$ and $\mathbf{o}_k$ as crossing points of the boundary of the convex hull of $\mathcal{A}_k$. Then we group these points in the set $\mathcal{C}^j=\{\mathbf{s},$ $\mathbf{i}_1,$ $\mathbf{i}_2,$ $\cdots,\mathbf{i}_j,$ $\mathbf{o}_j,$ $\mathbf{o}_{j-1},$ $\cdots,$ $\mathbf{o}_1,$ $\mathbf{g}\}$ and index its elements as follows:
\begin{equation}
\label{eq:4.0.2}
\begin{array}{l}
\mathbf{c}^j_n=\mathbf{i}_n,\ \ \ {\rm for}\ \ \ n=1,2,\cdots,j,\\
\mathbf{c}^j_n=\mathbf{o}_{2j+1-n},\ \ \ {\rm for}\ \ \ n=j+1,j+2,\cdots,2j,\\
\mathbf{c}^j_0=\mathbf{s},\ \ \mathbf{c}^j_{2j+1}=\mathbf{g}.
\end{array}
\end{equation}
where $\mathbf{c}^j_n$ is the $n$th element of the set $\mathcal{C}^j$; $\mathbf{s}$ and $\mathbf{g}$ are the starting and ending points of the trajectory. We also define $t_n$ as the time instant in which the drone passes by $\mathbf{c}^j_n$ and:
 \begin{equation}
\label{eq:4.0.3}
\tau_{n}\triangleq(t_{n+1}-t_{n}), \ \ n=0,1,\cdots,2j.
 \end{equation}
Then, we complete the first step in the top-down decomposition by assuming\footnote{This assumption holds in most cases as long as the quadrotor do not have to perform an abrupt change of direction within a short time. This assumption simplifies the optimization problem and will allow us to derive a suboptimal solution.} that the drone remains in the same region $\mathcal{A}_k$ during the whole period $t\in[t_{n-1},t_n)$ as it goes from $\mathbf{c}^j_{n-1}$ to $\mathbf{c}^j_n$. Such an assumption eliminates the dependence of the communications term in the cost function (\ref{eq:3.0.6}) on the position $\mathbf{p}(t)$ and makes it dependent on the time intervals characterized by $\left\{\tau_{n}\right\}_{n=0}^{2n}$ instead. This key step \emph{decouples the communications aspect from the control aspect} and makes the problem more tractable. By doing so, the argument of the communications term in (\ref{eq:3.0.6}) (i.e., the quantized version of the expected number of bits transmitted during the trajectory) is then approximated by:
\begin{equation}
\label{eq:4.0.4a}
\begin{array}{r}
T_{tx}\displaystyle\sum_{k=0}^{\lfloor\frac{t_f}{T}\rfloor}f_{_Q}(\mathbf{p}(kT)) 
\approx \frac{T_{tx}}{T}\left(\tau_{_j}R_j^Q+\displaystyle\displaystyle\sum_{k=0}^{j-1}\left(\tau_{_k}+\tau_{_{2j-k}}\right)R_k^Q\right)\\
\triangleq\mathcal{B}\left(\{\tau_{_n}\}_{n=0}^{2j}\right)
\end{array}
\end{equation}

Now, we constraint the drone to pass through $\{\mathcal{A}_k\}_{k=1}^{j}$, use the approximation (\ref{eq:4.0.4a}), and fix $\{\tau_n\}_{n=1}^{2j+1}$. Hence, we obtain the following set of optimum control problems:
\begin{equation}
\label{eq:4.0.5}
    \begin{array}{l}
\displaystyle \minimize_{\bar{\mathbf{u}}}\ \ \ \ \ \ \int_0^{t_f}\left\|\bar{\mathbf{u}}(t)\right\|_2^2\mathrm{d}t\\
{\rm s.t.}\\
\derT{{\zeta}_{\xx}}{t} = {A}_{\xx}\,{\zeta}_{\xx} +
{B}_{\xx}\bar{u}_{\xx},\ \
\derT{{\zeta}_{\yy}}{t} = {A}_{\yy}\,{\zeta}_{\yy} +
{B}_{\yy}\bar{u}_{\yy},\\
{\zeta_\xx(0)} =
\left[\begin{array}{cccc}
{s}_\xx & 0 & 0 & 0
\end{array}\right]^{\mathrm{T}}
,\ \
{\zeta_\yy(0)} =
\left[\begin{array}{cccc}
{s}_\yy & 0 & 0 & 0
\end{array}\right]^{\mathrm{T}}
,\\
{\zeta_\xx(t_f)} =
\left[\begin{array}{cccc}
\mathrm{g}_\xx & 0 & 0 & 0
\end{array}\right]^{\mathrm{T}}
,\ \
{\zeta_\yy(t_f)} =
\left[\begin{array}{cccc}
\mathrm{g}_\yy & 0 & 0 & 0
\end{array}\right]^{\mathrm{T}}
,\\
\mathbf{p}(t_n^j)=\mathbf{c}_n^j, \ \ n=0,1,\cdots,2j+1,\ j=0,1,\cdots,Q-1,\\
\end{array}
\end{equation}
The solution of (\ref{eq:4.0.5}) is a minimum norm control law $\bar{\mathbf{u}}_j^*$ that makes the drone to pass through all the points in $\mathcal{C}^j$ at time instances $\{t_n^j\}_{n=0}^{2j+1}$. Such optimum trajectory can be decomposed into $2j+1$ subtrajectories and due to Pontryagin's minimum principle \cite{b1} these subtrajectories must also be optimum trajectories between the points $\mathbf{c}_n^j$ to $\mathbf{c}_{n+1}^j$ for $n=0,1,\cdots,2j+1$. 

So, the second step in our top-down decomposition process is to further decompose (\ref{eq:4.0.5}) in order to be able to obtain the optimum control laws that generate the optimum subtrajectories mentioned above. This results in the following problem:
\begin{equation}
\label{eq:4.0.6}
    \begin{array}{l}
\displaystyle \minimize_{\bar{\mathbf{u}}}\ \ \ \ \ \ \int_{t_n}^{t_{n+1}}\left\|\bar{\mathbf{u}}(t)\right\|_2^2\mathrm{d}t\\
{\rm s.t.}\\
\derT{{\zeta}_{\xx}}{t} = {A}_{\xx}\,{\zeta}_{\xx} +
{B}_{\xx}\bar{u}_{\xx},\ \
\derT{{\zeta}_{\yy}}{t} = {A}_{\yy}\,{\zeta}_{\yy} +
{B}_{\yy}\bar{u}_{\yy},\\
{\zeta_\xx(t_{n})} =
\left[\begin{array}{cccc}
{c}_{n,x}^j & a_{n,x,1}^j & a_{n,x,2}^j & a_{n,x,3}^j
\end{array}\right]^{\mathrm{T}}
,\\
{\zeta_\yy(t_{n})} =
\left[\begin{array}{cccc}
{c}_{n,y}^j & a_{n,y,1}^j & a_{n,y,2}^j & a_{n,y,3}^j
\end{array}\right]^{\mathrm{T}}
,\\
{\zeta_\xx(t_{n+1})} =
\left[\begin{array}{cccc}
{c}_{n+1,x}^j & a_{n+1,x,1}^j & a_{n+1,x,2}^j & a_{n+1,x,3}^j
\end{array}\right]^{\mathrm{T}}
,\\
{\zeta_\yy(t_{n+1})} =
\left[\begin{array}{cccc}
{c}_{n+1,y}^j & a_{n+1,y,1}^j & a_{n+1,y,2}^j & a_{n+1,y,3}^j
\end{array}\right]^{\mathrm{T}}
,\\
n=0,1,\cdots,2j+1,\ \ j=0,1,\cdots, Q-1,\\
\end{array}
\end{equation}
where ${c}_{n,x}^j$ and ${c}_{n,y}^j$ are the $x$ and $y$ components of $\mathbf{c}_n^j$ respectively, while $\{a_{n,x,k}^j,a_{n,y,k}^j\}_{k=1}^3$ are parameters to be optimized in the next subsection. This concludes the decomposition phase. In the next section we solve all the problems presented in this section and integrate them to obtain a solution to the optimization problem (\ref{eq:3.0.7}).

\subsection{Bottom-up Integration} 
\label{sec:Solution:integration}
After having decomposed the optimization problem (\ref{eq:3.0.7}) in various sub-problems in the previous subsection, we proceed now to solve these sub-problems and combine their solutions. We start by analytically solving (\ref{eq:4.0.6}) (see appendix \ref{Appendix:B}). For $t_{n} \leq t < t_{n+1}$ we have that:
\begin{equation}
\label{eq:4.0.7}
{u}_{*i}(t) =\mcal{F}_i(t_{n+1}-t)\,\eu{W}_{\tau_{n}i}^{-1}\,(\zeta_i(t_{n+1})- \exp\left(A_i\tau_n\right)\zeta_i(t_{n})),\\
\end{equation}
where $i\in\{\xx,\yy\}$, $\exp\left(A_y\tau_{n}\right)$ is the exponential matrix of $A_y\tau_{n}$ and $\eu{W}_{\tau_{n}x}$ is given in Appendix \ref{Appendix:B}. Note that (\ref{eq:4.0.7}) is a minimum norm control law that takes the drone from $\mathbf{c}_n^j$ at time instant $t_n$ to $\mathbf{c}_{n+1}^j$ at time instant $t_{n+1}$ and whose state vectors $\zeta_\xx(t_{n})$, $\zeta_\yy(t_{n})$, $\zeta_\xx(t_{n+1})$ and $\zeta_\yy(t_{n+1})$ depend on the set of points $\mathcal{C}^j$ and on $\{a_{n,x,k}^j,a_{n,y,k}^j\}_{k=1}^3$ and $\{a_{n+1,x,k}^j,a_{n+1,y,k}^j\}_{k=1}^3$ as described in (\ref{eq:4.0.6}). Note that by applying the control law ${u}_{*i}(t) $ (\ref{eq:4.0.7}) during $t\in[t_n,t_{n+1}]$ the states $\zeta_\xx(t)$ and $\zeta_\yy(t)$ are then given by (\ref{eq:4.0.7a}) where $M_i(t;t_n)$ is given by (\ref{eq:4.0.7b}).

\begin{figure*}[t]
\begin{eqnarray}
\zeta_i(t)&=&\exp(A_it)M_i(t;t_n)(M_i(t_{n+1};t_n))^{-1}(\exp(A_it_{n+1}))^{-1}\,(\zeta_i(t_{n+1})- \exp\left(A_i\tau_n\right)\zeta_i(t_{n}))+\exp\left(A_i(t-t_n)\right)\zeta_i(t_{n}),\nonumber\\
&&t\in[t_n,t_{n+1}]
\label{eq:4.0.7a}
\end{eqnarray}
\begin{eqnarray}
M_i(t;t_n) & = & B_iB_i^\mathrm{T}(t-t_n)-\frac{1}{2}\left(A_iB_iB_i^\mathrm{T}+B_iB_i^\mathrm{T}A_i^\mathrm{T}\right)(t^2-t_n^2)\nonumber\\
&+&\frac{1}{3}\left(\frac{1}{2}A_i^2B_iB_i^\mathrm{T}+A_iB_iB_i^\mathrm{T}A_i^\mathrm{T}+\frac{1}{2}B_iB_i^\mathrm{T}(A_i^\mathrm{T})^2\right)(t^3-t_n^3)\nonumber\\
&-&\frac{1}{4}\left(\frac{1}{6}A_i^3B_iB_i^\mathrm{T}+\frac{1}{2}A_i^2B_iB_i^\mathrm{T}A_i^\mathrm{T}+\frac{1}{2}A_iB_iB_i^\mathrm{T}(A_i^\mathrm{T})^2+\frac{1}{6}B_iB_i^\mathrm{T}(A_i^\mathrm{T})^3\right)(t^4-t_n^4)\nonumber\\
&+&\frac{1}{5}\left(\frac{1}{6}A_i^3B_iB_i^\mathrm{T}A_i^\mathrm{T}+\frac{1}{4}A_i^2B_iB_i^\mathrm{T}(A_i^\mathrm{T})^2+\frac{1}{6}A_iB_iB_i^\mathrm{T}(A_i^\mathrm{T})^3\right)(t^5-t_n^5)\nonumber\\
&-&\frac{1}{6}\left(\frac{1}{12}A_i^3B_iB_i^\mathrm{T}(A_i^\mathrm{T})^2+\frac{1}{12}A_i^2B_iB_i^\mathrm{T}(A_i^\mathrm{T})^3\right)(t^6-t_n^6)+\frac{1}{7}\left(\frac{1}{36}A_i^3B_iB_i^\mathrm{T}(A_i^\mathrm{T})^3\right)(t^7-t_n^7),
\label{eq:4.0.7b}
\end{eqnarray}
\vspace*{4pt}
\rule[0.5ex]{1\textwidth}{1pt}
\end{figure*}

For a particular value of $j$, we solve (\ref{eq:4.0.5}) by first concatenating the set of the $2j+1$ parameterized control laws resulting from (\ref{eq:4.0.7}) which are solutions to (\ref{eq:4.0.6}). We denote such a concatenated control law by $\bar{\mathbf{u}}_j^*(t,\bm{\alpha}^j,\{\tau_n\}_{n=0}^{2j},\mathcal{C}^j)$ which is parameterized over $\bm{\alpha}^j$ (as well as over $\{\tau_n\}_{n=0}^{2j}$ and $\mathcal{C}^j$): 
\begin{equation}
\begin{array}{r}
\bm{\alpha}^j=
\left[\{\bm{\alpha}_\xx^j\}^\mathrm{T},\ \{\bm{\alpha}_\yy^j\}^\mathrm{T}\right]^\mathrm{T}\\
\bm{\alpha}_i^j=\left[\{\bm{\alpha}_{0,i}^j\}^\mathrm{T},\ \{\bm{\alpha}_{1,i}^j\}^\mathrm{T},\cdots,\ \{\bm{\alpha}_{2j,i}^j\}^\mathrm{T}\right]^\mathrm{T}\\
\bm{\alpha}_{n,i}^j=\left[a_{n,i,1}^j,\ a_{n,i,2}^j\ \ a_{n,i,3}^j \right]^\mathrm{T}\\
\end{array}
\end{equation}
where  $i\in\{\xx,\yy\}$ and $\{a_{n,x,k}^j,a_{n,y,k}^j\}_{k=1}^3$ are the state parameters in (\ref{eq:4.0.6}).

After including $\bar{\mathbf{u}}=\bar{\mathbf{u}}^*_j(t,\bm{\alpha}^j,\{\tau_n\}_{n=0}^{2j},\mathcal{C}^j)$ (defined above) in the optimization problem (\ref{eq:4.0.5}), the latter becomes (for a particular value of $j$):
\begin{equation}
\label{eq:4.0.7dd}
    \begin{array}{l}
\displaystyle \minimize_{{\bm{\alpha}}^j}\ \ \ \ \ \ \int_0^{t_f}\left\|\bar{\mathbf{u}}^*_j\left(t,\bm{\alpha}^j,\{\tau_n\}_{n=0}^{2j},\bm{c}^j\right)\right\|_2^2\mathrm{d}t
\end{array}
\end{equation}
where:
\begin{equation}
\label{eq:4.0.8}
\begin{array}{r}
\displaystyle\int_0^{t_f}\left\|\bar{\mathbf{u}}_j^*\left(t,\bm{\alpha}^j,\{\tau_n\}_{n=0}^{2j},\bm{c}^j\right)\right\|_2^2\mathrm{d}t=\\
\displaystyle\sum_{n=0}^{2j}\left(\left[
\begin{array}{c}
c^j_{n+1,\xx}\\
\bm{\alpha}_{n+1,\xx}^j
\end{array}
\right]
-\exp\left(A_\xx\tau_{n}\right) \left[
\begin{array}{c}
c^j_{n,\xx}\\
\bm{\alpha}_{n,\xx}^j
\end{array}
\right]\right)^\mathrm{T}\eu{W}_{\tau_{n}\xx}^{-1}\\
\left(\left[
\begin{array}{c}
c^j_{n+1,\xx}\\
\bm{\alpha}_{n+1,\xx}^j
\end{array}
\right]-\exp\left(A_\xx\tau_{n}\right) \left[
\begin{array}{c}
c^j_{n,\xx}\\
\bm{\alpha}_{n,\xx}^j
\end{array}
\right]\right)\\
+\displaystyle\sum_{n=0}^{2j}\left(
\left[
\begin{array}{c}
c^j_{n+1,\yy}\\
\bm{\alpha}_{n+1,\yy}^j
\end{array}
\right]
-\exp\left(A_\yy\tau_{n}\right) \left[
\begin{array}{c}
c^j_{n,\yy}\\
\bm{\alpha}_{n,\yy}^j
\end{array}
\right]\right)^\mathrm{T}
\eu{W}_{\tau_{n}\yy}^{-1}\\
\left(\left[
\begin{array}{c}
c^j_{n+1,\yy}\\
\bm{\alpha}_{n+1,\yy}^jt
\end{array}
\right]
-\exp\left(A_\yy\tau_{n}\right) 
\left[
\begin{array}{c}
c^j_{n,\yy}\\
\bm{\alpha}_{n,\yy}^j
\end{array}
\right]\right)
\end{array}
\end{equation}
To solve (\ref{eq:4.0.7}), we first note that we need $\bm{\alpha}_{1,\xx}^j=\bm{\alpha}_{1,\yy}^j=\bm{\alpha}_{2j+1,\xx}^j=\bm{\alpha}_{2j+1,\yy}^j=\bm{0}$ in order to satisfy the initial and final states of the drone required by  (\ref{eq:4.0.5}). The rest of the parameters  $\left\{\bm{\alpha}_{n,\xx},\bm{\alpha}_{n,\yy}\right\}_{n=1}^{2j-1}$ are optimized by solving (\ref{eq:4.0.7}). Note that the cost function is quadratic with respect to these parameters as the matrix $\eu{W}_{\tau_{n}x}$ is independent of the ${\bm \alpha}$'s. Therefore (\ref{eq:4.0.7}) is a convex optimization problem and has a unique solution. This solution can be obtained by first calculating analytically the gradient of (\ref{eq:4.0.8}) with respect to the state vectors $\bm{\alpha}_{1,\xx}^j=\bm{\alpha}_{1,\yy}^j=\bm{\alpha}_{2j+1,\xx}^j=\bm{\alpha}_{2j+1,\yy}^j=\bm{0}$, and then set this gradient to zero to solve the resulting matrix equations numerically. 

After having optimized $\bm{\alpha}^j$, we plug the optimized vectors into the concatenated control law $\bar{\mathbf{u}}_j^*\left(t,\bm{\alpha}^j,\{\tau_n\}_{n=0}^{2j},\mathcal{C}^j\right)$. We refer to the resulting control as $\bar{\mathbf{u}}_j^*\left(t,\{\tau_n\}_{n=0}^{2j},\mathcal{C}^j\right)$. This new control law $\bar{\mathbf{u}}^*_j\left(t,\{\tau_n\}_{n=0}^{2j},\mathcal{C}^j\right)$ is the solution to the optimization problem (\ref{eq:4.0.5}) (parameterized over $j$).

Now, we remind the reader that the optimization problem (\ref{eq:4.0.5}) was derived from the optimization problem (\ref{eq:3.0.7}) by fixing the set of points $\mathcal{C}^j$ as well as the time duration sequence $\{\tau_n\}_{n=0}^{2j}$ and using the approximation (\ref{eq:4.0.4a}) for the communications term in the cost function.

To continue with our bottom-up integration, we use the control law $\bar{\mathbf{u}}_j^*\left(t,\{\tau_n\}_{n=0}^{2j},\mathcal{C}^j\right)$ derived above and then we use the approximation (\ref{eq:4.0.4a}) of the cost function in (\ref{eq:3.0.7}). By doing so, the cost function (\ref{eq:3.0.6}) becomes:  
\begin{eqnarray}
\label{eq:4.0.8c}
\mathcal{J}'\left(\bar{\mathbf{u}},\lambda,t_f\right)&\approx&
\lambda \int_0^{t_f}\left\|\bar{\mathbf{u}}^*\left(t,\{\tau_n\}_{n=0}^{2j},\mathcal{C}^j\right)\right\|^2\mathrm{d}t\nonumber\\
&+&(1-\lambda)w\left(\mathcal{B}\left(\{\tau_n\}_{n=0}^{2j}\right)\right)\nonumber\\
&\triangleq&\mathcal{J}''\left(\{\tau_n\}_{n=0}^{2j},\mathcal{C}^j\right)
\end{eqnarray}
where $\mathcal{B}\left(\{\tau_n\}_{n=0}^{2j}\right)$ is defined in (\ref{eq:4.0.4a}) and the optimization problem (\ref{eq:3.0.7}) becomes: 
\begin{equation}
\label{eq:4.0.9}
    \begin{array}{l}
\displaystyle \minimize_{\{\tau_n\}_{n=0}^{2j}, \{\beta_n\}_{n=1}^{2j}}\ \ \ \ \ \ 
\mathcal{J}''\left(\{\tau_n\}_{n=0}^{2j},\mathcal{C}^j\right)\\
{\rm s.t.}\\
\mathcal{C}^j=\{\bm{c}^j_0, \bm{c}^j_1,\cdots,\bm{c}^j_{2j+1}\},\\
\bm{c}^j_n=r_n[\cos(\beta_n)\ \ \sin(\beta_n)],
\end{array}
\end{equation}
In the last line of the constraints, we have the points $\bm{c}^j_n$ in polar coordinates and due to their definition and of the regions $\{\mathcal{A}_k\}_{k=0}^j$ in (\ref{eq:3.0.8})) the radii $\{r_n\}_{n=1}^{2j}$ are known and their values can be obtained numerically from (\ref{eq:3.0.8}). So the locations of the points $\{\bm{c}^j_n\}_{n=1}^{2j}$ are fully determined by the angles $\{\beta_n\}_{n=1}^{2j}$.

Now, since the domain of the variables $\beta_n$ and $\tau_n$ is bounded, we can solve (\ref{eq:4.0.9}) using simulated annealing (SA) \cite{b2}. Let us call $\bar{\mathbf{u}}^*_j\left(t\right)$ the control law produced by solving (\ref{eq:4.0.9}). 

Since $\bar{\mathbf{u}}^*_j\left(t\right)$ is still parameterized over the number $j$, the final step in obtaining a suboptimal solution for (\ref{eq:3.0.7}) is to solve (\ref{eq:4.0.9}) for $j=0,1,\cdots,J$, compare the values of the cost function obtained with each value of $j$, and select the control law  $\bar{\mathbf{u}}^*_j\left(t\right)$ that produces the minimum value of the cost function in (\ref{eq:4.0.9}). It is worth pointing out that, as mentioned above, the solution obtained by this method is a sub-optimal solution for (\ref{eq:3.0.7}) mainly due to the approximation in (\ref{eq:4.0.4a}).

\section{Obstacles Avoidance}
\label{sec:Obstacles}

The solution presented in section \ref{sec:Solution} does not consider obstacles so in this section we show how to slightly modify our method to take them into account while increasing the computational complexity just slightly. To achieve this, let us start by considering a single circular obstacle of center $\mathbf{q}_O$ and radius $r_o$.  As mentioned in section \ref{sec:Solution} the robot will pass through points $\{\mathbf{c}_n^j\}_{n=0}^j$ in ascending order. Then, let us assume that the robot's path can be approximated by a piecewise linear path passing through the same points $\{\mathbf{c}_n^j\}_{n=0}^j$. If we do so then in order to allow the robot to avoid the obstacle it suffices to ensure that the distance between the center of the obstacle (i.e., $\mathbf{q}_O$) and each linear segment composing the piece-wise linear path mentioned above is greater than $r_o$.

This can be easily done by adding a penalization term to the cost function (\ref{eq:4.0.9}) inspired from the concept of artificial potential fields \cite{b11} which is used for obstacle avoidance problems. The penalization term that we propose is:
\begin{equation}
\label{eq:5.0.1}
P_{obs}=K_1\sum_{n=0}^{j-1}\exp\left( K_2\frac{\|\mathbf{q}_o-\mathbf{a}^*(n,n+1)\|-r_o}{r_o} \right)
\end{equation}
where $K_1>0$, $K_0>0$ are design parameters; $\mathbf{a}^*(n,n+1)$ is the closest point to the obstacle center $\mathbf{q}_o$ belonging to the linear segment $\theta\mathbf{c}_n+(1-\theta)\mathbf{c}_{n+1}$ with $\theta\in[0,1]$. Using linear algebra it is possible to show that 
\begin{equation}
\label{eq:5.0.2}
\mathbf{a}^*(n,n+1)=\theta^*\mathbf{c}_n+(1-\theta^*)\mathbf{c}_{n+1}
\end{equation}
where:
\begin{equation}
\label{eq:5.0.3}
\theta^*(n,n+1)=\max\left(\min\left(\frac{(\mathbf{c}_{n}-\mathbf{c}_{n+1})^\mathrm{T}(\mathbf{q}_o-\mathbf{c}_{n+1})}{\|\mathbf{c}_{n+1}-\mathbf{c}_{n}\|^2},1\right),0\right)
\end{equation}
To take into consideration more obstacles it suffices to add similar penalization terms per obstacle and in order to consider obstacles with more complex shapes we could represent them as the union of several small circular obstacles.

Finally, we have to mention that the purpose of this section is to show that it is possible to integrate obstacle avoidance mechanisms into our trajectory planning method. This means that other and more sophisticated obstacle avoidance techniques could also be adapted to our method.

\section{Simulations}
\label{sec:Simulations}
In this section we present some simulation results to gain more insight into the trajectories produced by our method for solving the {\it minimum energy-maximum data} problem described in section \ref{sec:ProblemStatement}. So, the function $w(\cdot)$ in (\ref{eq:3.0.6}) takes the form given by (\ref{eq:2.0.3}).

For the drone we select the parameters presented in \cite{r1} which represent a real drone. For the channel model, we select $\alpha=2$ and 1 dB for the variance of the shadowing\footnote{These are realistic values close to the ones used in the literature for similar applications \cite{r22}.}.

Now, for the communications system, we assume that during each duplexing period the drone transmits $N_s$ symbols using the modulation schemes 4-QAM, 16-QAM and 64-QAM (see \cite{b9}) which results in the following bit rates 2$R_s$, 4$R_s$ and 6$R_s$  where $R_s$ is the symbol transmission rate; we also select the thresholds $\gamma_j$ such that the bit error probability is always lower than $10^{-3}$; finally, we set $10\log_{10}(P/\sigma^2)=40$ dB. 

\begin{figure}[h]
\vspace{-4mm}
\centerline{\includegraphics[clip, trim ={5mm 1mm 5mm 1mm}, height=4cm,width=8cm]{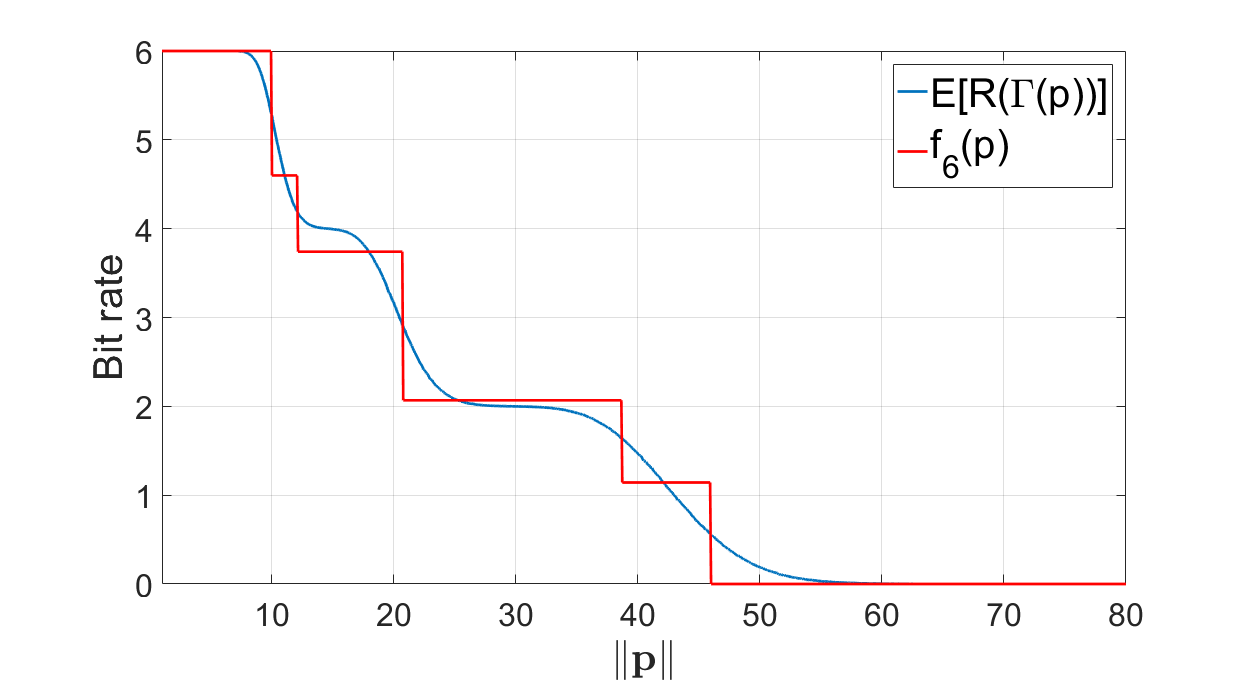}}
\vspace{-3mm}
\caption{Comparison of $\mathbb{E}[R(\Gamma(\mathbf{p}))]$ with its quantized version $f_Q(\mathbf{p})$ with $Q=6$.}
\label{Figure0}
\end{figure}

Regarding the quantized function $f_Q(\mathbf{p}(kT))$ in (\ref{eq:3.0.6}) with $Q=6$ quantization levels (see Fig. \ref{Figure0}) which results into six different regions $\{\mathcal{A}_k\}_{j=0}^5$ (see \ref{eq:3.0.8}) whose boundaries are marked by gray circles in Fig. \ref{Figure1}. The quantization shown $f_6(\cdot)$ according to the method described in (\ref{Appendix:C}) The AP is located at the origin, shown in black in Fig. \ref{Figure1}. The starting and goal points are set to $[75\ \ 0]$ and $80[\cos\left(\frac{5\pi}{9}\right)\ \ \sin\left(\frac{5\pi}{9}\right)]$ respectively (see Fig. \ref{Figure1} ) while $t_f=100$s. The simulations are then performed using MATLAB.

In Figs. \ref{Figure1}-\ref{Figure3} we show the paths of the optimized trajectories for various values of $\lambda$: $\lambda=1$ (black); $\lambda=0.98$ (blue);  $\lambda=0.8$ (red);  $\lambda=0.5$ (green);  $\lambda=0.2$ (yellow);  $\lambda=0.1$ (magenta). The case of $\lambda=1$ results in the drone moving on a straight line from the starting point to the goal point in a time $t_f$ using minimum energy. We refer to such trajectory as the minimum energy trajectory (MET) and will serve as comparison.

\begin{figure}[h]
\vspace{-5mm}
\centerline{\includegraphics[clip, trim ={2mm 1mm 2mm 1mm}, height=4.5cm,width=8cm]{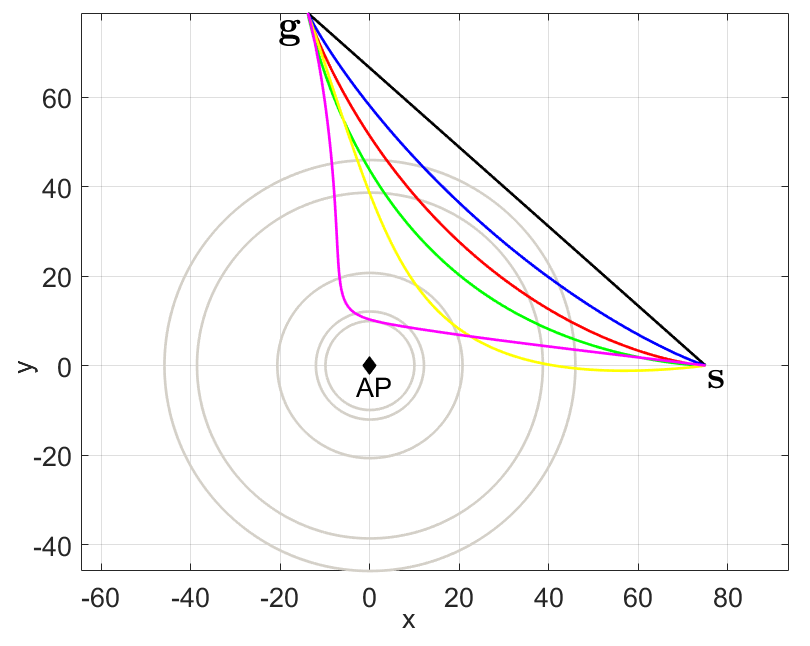}}
\vspace{-5mm}
\caption{Optimum paths for different values of $\lambda$.}
\label{Figure1}
\end{figure}
\begin{figure}[h]
\centerline{\includegraphics[clip, trim ={4mm 1mm 2mm 1mm}, height=4.5cm,width=8cm]{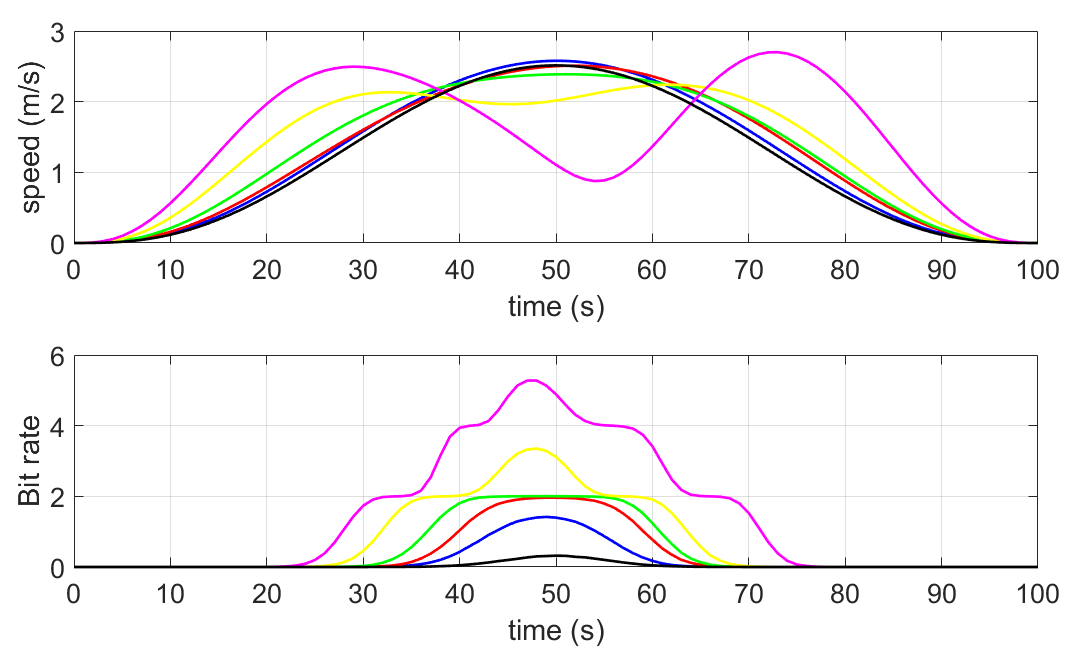}}
\vspace{-3mm}
\caption{Speed (top) and transmission (bottom) profiles  for different values of $\lambda$.}
\label{Figure2}
\end{figure}

When the value of $\lambda$ is high (close to 1) the path of the optimum trajectory deviates only slightly from the MET's  path, see Fig. \ref{Figure1}. This results in a slightly curved path attracted to the location of the APP. Meanwhile, the speed profile remains very similar to one of the MET, see Fig. \ref{Figure2} (top), but the profile of the bit rate transmission is significantly different, see Fig. \ref{Figure2} (bottom).   

Note that as the value of $\lambda$ decreases the shape of the speed profile transforms from having a single maximum to having two maxima, see Fig. \ref{Figure2} (top). This is due to the fact that when $\lambda$ is low the drone wants to maximize the amount of data transmitted. This is achieved by going as fast as possible to the regions with higher bit rate, then spending as much time there in order to take advantage of the higher bit rate. As a consequence this leaves the drone little time to reach the goal point $\mathbf{g}$ in time and hence it has to accelerate in order to reach it in time. This is clearly translated into the speed profile (see Fig. \ref{Figure2} (top)) by producing two local maxima separated by a local minima. In addition this is also reflected in the shape of the path which tends to be the concatenation of two paths, see Fig. \ref{Figure1},: one (almost) straight path connecting directly the starting point $\mathbf{s}$ to an area with high bit rate near the AP and second one (also almost straight) connecting with the goal point $\mathbf{g}$.

In table. \ref{tab:1} we observe: in the first column the ratio of the energy consumed by the robot by executing the optimum trajectory over the amount of energy spent by the robot while following the MET;  in the second column the ratio of the average data transmission (measured by simulations) divided by average data transmitted by the drone during the MET; the third column presents the average data transmitted during the optimized trajectory calculated using the approximation (\ref{eq:4.0.4a}) which was used during the optimization procedure; in the fourth column we present the same average data transmitted during the trajectory but this time measured by simulations. The values presented on the third and fourth columns are normalized over $R_sT_{tx}/T$. 

By observing the first two columns of table \ref{tab:1} we observe that the energy consumption is an increasing function of $\lambda$ while the number of bits transmitted is a decreasing function of $\lambda$. Thus confirming that the parameter $\lambda$ controls the compromise between minimizing the energy spent in motion and maximizing the amount of data transmitted. It is worth noting that by setting $\lambda$ close to 1 we can significantly increase the number of bits transmitted with a relatively small increase in the energy consumption (see $\lambda=0.95$).

From table \ref{tab:1} we also note the expected number of bits transmitted calculated using the approximation (\ref{eq:4.0.4a}) and the expected number of bits transmitted measured by simulations are quite similar. Hence, this validates the approximation used during the optimization  (\ref{eq:4.0.4a}) used in the approximation as well as the utilization of the quantized function $f_Q(\mathbf{p})$ to approximate average bit rate.
\begin{table}[t]
\global\long\def\arraystretch{1.3}
 \caption{Performance of optimum trajectories}
\label{tab:1} \centering \global\long\def\arraystretch{1.2}
\begin{tabular}{|c|c|c|c|c|}
\hline
$\lambda$  & Energy   & Transmission  & Transmitted & Transmitted  \tabularnewline
  & ratio  &  ratio &  data (approx)&  data (measured) \tabularnewline
\hline
\hline  
0.98  &1.2587 &  4.9586& 20.58 &  20.60 \tabularnewline
\hline
0.95  & 1.7408  & 8.9156 &37.06& 37.09 \tabularnewline
\hline
0.9  &1.7532 &  9.0169&37.50 &  37.46 \tabularnewline
\hline
0.8  &1.8059  & 9.2104 &38.20 &  38.21\tabularnewline
\hline
0.7  &1.9189&   9.4478 &39.26 & 39.27\tabularnewline
\hline
0.6  &2.5208&   11.5680 &48.03 &  48.06 \tabularnewline
\hline
0.5  &2.6014&   11.7491 &48.73 &  48.76\tabularnewline
\hline
0.4  &6.1456&   17.6298 &73.25 &  73.26 \tabularnewline
\hline
0.3  &6.7555&   17.7892 &74.02 &  74.01 \tabularnewline
\hline
0.2  &6.7808&   17.9796 &74.66 &  74.63\tabularnewline
\hline
0.1  &30.2683&  34.1361 &141.74 &  141.70\tabularnewline
\hline
\end{tabular}
\end{table}

Now, we show obstacle avoidance capability of our method. To do this we add a circular obstacle of radius $r_o=5$ m at the point $[60\ 0]$ (represented as a red circle in Fig. \ref{Figure3}) which lies into the path of the optimum trajectory with $\lambda=0.6$  (green) which does not takes into account the obstacle. Then, we optimize the trajectory, with $\lambda=0.6$, while taking into account the obstacle using the method described in section \ref{sec:Obstacles} with coefficients $K_1=1000$ and $K_2=100$. The corresponding path is shown in blue in Fig. \ref{Figure3}. We observe from the zoomed portion in Fig. \ref{Figure3} that indeed the new optimum path (shown in blue) is able to successfully avoid the obstacle  (shown in red) and hence shows that the trajectory optimization method proposed in this article is also able to consider obstacles.

\begin{figure}[h]
\centerline{\includegraphics[clip, trim ={2mm 1mm 2mm 1mm}, height=5cm,width=8cm]{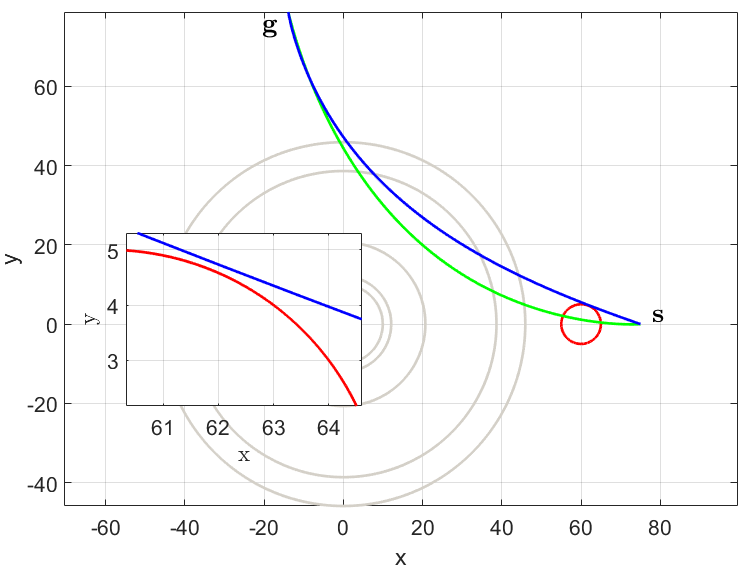}}
\caption{Paths corresponding to the optimum trajectories ($\lambda=0.6$) considering the obstacle (blue) and without considering the obstacle (green). The magnified image shows the vicinity of the obstacle.}
\label{Figure3}
\end{figure}

Finally, as mentioned in the introduction, many adaptive trajectory planners that take into account communications and energy aspects require a predetermined path to follow and modify according to channel measurements obtained while following the path. Therefore our technique could be used to provide such a path and therefore complement such algorithms.

\section{Conclusion}
\label{sec:Conclusions}
We have developed a technique to optimize a predetermined trajectory (path and velocity profile) that takes a robot from a certain initial point to a goal point in finite time while taking into account communications and energy constraints whose relative importance is determined by the design parameter $\lambda$. We were able to develop such a technique while taking into account the full dynamic model of the robot as well as a realistic model for the wireless channel, and requiring only a limited knowledge of the wireless channel, consisting only of its first order statistics. Results show that using our optimum trajectories and by selecting an appropriate value for $\lambda$, it is possible to significantly increase the number of bits transmitted during the trajectory. Our predetermined trajectory can be used jointly with other adaptive trajectory planners which require a predetermined path. Future work will focus on developing an experimental prototype to test the method developed in this paper in real environments.
\section*{Acknowledgements}
This work is partly funded by USAID under the grant agreement number 2000007744,  and the CNRS associated laboratory DATANET.
\begin{appendices}
\section{}
\label{Appendix:A}
In this appendix we introduce a new linearization method which we apply to the quadrotor model described by equations (\ref{eq:QR_model})-(\ref{eq:Inp_vec}). Let us first define the following state vectors:
\begin{equation}
\label{eq:StVar}
\begin{array}{l}
{x}_{\xx} \triangleq
\left[\begin{array}{cccc}
{\xx} & \derT{\xx}{t} & {\theta} & \derT{\theta}{t}
\end{array}\right]^{\mathrm{T}},
\ \
{x}_{\zz} \triangleq
\left[\begin{array}{cc}
{\zz} & \derT{\zz}{t}
\end{array}\right]^{\mathrm{T}},
\\
{x}_{\yy} \triangleq
\left[\begin{array}{cccc}
{\yy} & \derT{\yy}{t} & {\phi} & \derT{\phi}{t}
\end{array}\right]^{\mathrm{T}},
\ \
{x}_{\psi} \triangleq
\left[\begin{array}{cc}
{\psi} & \derT{\psi}{t}
\end{array}\right]^{\mathrm{T}},
\end{array}
\end{equation}
and then apply the following state feedback:
\begin{equation}
\label{eq:PreStFeed}
\left[\begin{array}{c}
\mathrm{c}(\phi)\mathrm{c}(\theta){u}_{\zz}/{m} \\
{{u}_{\psi}}/{I_\zz}
\end{array}\right] =
\left[\begin{array}{c}
F_{\zz}\,{x}_{\zz} \\
F_{\psi}\,{x}_{\psi}
\end{array}\right] +
\left[\begin{array}{c}
{g}
\\
0
\end{array}\right]
,
\end{equation}
\begin{equation}
\label{eq:DefStFeed}
F_{\zz} \triangleq
\left[\begin{array}{cc}
-a_{\zz_2} & -a_{\zz_1}
\end{array}\right],
\ \
F_{\psi} \triangleq
\left[\begin{array}{cc}
-a_{\psi_2} & -a_{\psi_1}
\end{array}\right],
\end{equation}
such that the polynomials $\pi_{\zz}(\lp) = {\lp}^{2} + a_{\zz_1}{\lp} + a_{\zz_2}$
and $\pi_{\psi}(\lp) = {\lp}^{2} + a_{\psi_1}{\lp} + a_{\psi_2}$ are Hurwitz, i.e., all their roots have negative real parts. By applying such a feedback the model (\ref{eq:QR_model})-(\ref{eq:QR_model2}) becomes:
\begin{equation}
\label{eq:EcEdo_syst}
\begin{array}{l}
\derT{{x}_{\psi}}{t} = {A}_{\psi}\,{x}_{\psi},
\ \
\derT{{x}_{\zz}}{t} = {A}_{\zz}\,{x}_{\zz} ,
\\
\derT{{x}_{\xx}}{t} = {A}_{\xx}\,{x}_{\xx} + {B}_{\xx}\,{u}_{\xx} + {S}_{\xx}\,{q}_{\xx,\theta}(x),
\ \
{y}_{\xx} = {C}_{\xx}\,{x}_{\xx},
\\
\derT{{x}_{\yy}}{t} = {A}_{\yy}\,{x}_{\yy} + {B}_{\yy}\,{u}_{\yy} + {S}_{\yy}\,{q}_{\yy,\phi}(x),
\ \
{y}_{\yy} = {C}_{\yy}\,{x}_{\yy},
\end{array}
\end{equation}
where:
\begin{equation}
\label{eq:EcEdo_syst_mat2}
\begin{array}{c}
{A}_{\zz} =
\left[\begin{array}{cc}
0 & 1 \\
-a_{\zz_2} & -a_{\zz_1}
\end{array}\right]
, \
{A}_{\psi} =
\left[\begin{array}{cc}
0 & 1 \\
-a_{\psi_2} & -a_{\psi_1}
\end{array}\right]
,
\\ \noalign{\smallskip}
{A}_{\xx} =
\left[\begin{array}{cccc}
0 & 1 & 0 & 0 \\
0 & 0 & g & 0 \\
0 & 0 & 0 & 1 \\
0 & 0 & 0 & 0
\end{array}\right]
,
\ \
{A}_{\yy} =
\left[\begin{array}{ccrc}
0 & 1 & 0 & 0 \\
0 & 0 & -g & 0 \\
0 & 0 & 0 & 1 \\
0 & 0 & 0 & 0
\end{array}\right]
,
\end{array}
\end{equation}

\begin{equation}
\label{eq:EcEdo_syst_mat3}
\begin{array}{c}
{B}_{\xx} = {B}_{\yy} =
\left[\begin{array}{c}
0 \\ 0 \\ 0 \\ {\ell}/{I}
\end{array}\right]
, \
{C}_{\xx} = {C}_{\yy} =
\left[\begin{array}{c}
1 \\ 0 \\ 0 \\ 0
\end{array}\right]^{T}
\\ \noalign{\smallskip}
{S}_{\xx} =
\left[\begin{array}{cccc}
0 & g & 0 & 0\\ 
0 & 0 & 0 & 1\\ 
\end{array}\right]^\mathrm{T}, \
{S}_{\yy} =
\left[\begin{array}{rccc}
0 & -g & 0 & 0\\ 
0 & 0 & 0 & 1\\ 
\end{array}\right]^\mathrm{T}
;
\end{array}
\end{equation}
${q}_{\xx,\theta}(x)$ $=$
$\left[\begin{array}{cc}
{q}_{\xx}(x) & {q}_{\theta}(x)
\end{array}\right]^{\mathrm{T}}$ and
${q}_{\yy,\phi}(x)$ $=$
$\left[\begin{array}{cc}
{q}_{\yy}(x) & {q}_{\phi}(x)
\end{array}\right]^{\mathrm{T}}$
are uncertainty vectors defined as follows:
\begin{equation}
\label{eq:EDO_syst-q}
\begin{array}{rcl}
q_\xx &\triangleq&
\left(
\frac{\mathrm{s}(\phi)\mathrm{s}(\psi)}
{\mathrm{c}(\phi)\mathrm{c}(\theta)}
+
\frac{\mathrm{s}(\theta)\mathrm{c}(\psi)}
{\mathrm{c}(\theta)}
\right)\left(
1 -
\frac{a_{\zz_2}}{g}\der{\zz}{t} - \frac{a_{\zz_1}}{g}{\zz}
\right) - \theta
,
\\ \noalign{\smallskip}
q_\yy &\triangleq&
\left(
\frac{\mathrm{s}(\theta)\mathrm{s}(\psi)}
{\mathrm{c}(\theta)}
-
\frac{\mathrm{s}(\phi)\mathrm{c}(\psi)}
{\mathrm{c}(\phi)\mathrm{c}(\theta)}
\right)\left(
-1 +
\frac{a_{\zz_2}}{g}\der{\zz}{t} + \frac{a_{\zz_1}}{g}{\zz}
\right) - \phi
,
\\ \noalign{\smallskip}
q_\phi &\triangleq&
-(I_\zz/I - 1)\,\derT{\theta}{t}\,\derT{\psi}{t} -
({J}/{I})\,\derT{\theta}{t}\,{q}_{w}
,
\\ \noalign{\smallskip}
q_\theta &\triangleq&
(I_\zz/I - 1)\,\derT{\phi}{t}\,\derT{\psi}{t} +
({J}/{I})\,\derT{\phi}{t}\,{q}_{w}
.
\end{array}
\end{equation}
Now, let us note that:\footnote{
\label{fn:q_w}
Note that,
${x}$ $=$ $0$ and
$\left[\begin{array}{cccc}
{u}_{\xx} & {u}_{\yy} & {u}_{\zz} & {u}_{\psi}
\end{array}\right]^{\mathrm{T}}$ $=$
$\left[\begin{array}{cccc}
0 & 0 & mg & 0
\end{array}\right]^{\mathrm{T}}$
imply (\textit{cf.} (\ref{eq:Inp_vec})):
$\omega_1^2$ $=$
$\omega_2^2$ $=$
$\omega_3^2$ $=$
$\omega_4^2$ $=$
$mg/(4\kappa_b)$, thus:
$q_w$ $=$ $0$.
}
${q}_{i}(0)$ $=$ $0$ and
$\left[\depT{{q}_{i}(x)}{{x}}\right]_{{x} = {0} \choose \bar{u}_{i} = {0}} = 0$,
$i$ $\in$ $\{{\xx},\,{\yy}\}$.
Hence ${q}_{i}(x)$,
$i$ $\in$ $\{{\xx},\,{\yy}\}$,
are indeed nonlinear perturbation signals and the
$\Sigma(A_{i}, B_{i}, C_{i})$,
$i$ $\in$ $\{{\xx},\,{\yy}\}$,
are linearized state descriptions of (\ref{eq:QR_model}) (after the state feedbacks (\ref{eq:PreStFeed})).

Following \cite{BoBlMaAzSa:17,BlBoSaMaAz:19}, we propose the change of variable\footnote{
At first, we solve the algebraic equation: $AM+BX=\Id{}$, and then, we set: 
$\zeta=x + M\left(S + \sum_{i=1}^{3}M^{i}S\psT{i}\right)q$. In our case we have: $
M_{\xx} = \left[\begin{array}{cccc}
0&0&0&0\\
1&0&0&0\\
0&1/g&0&0\\
0&0&1&0
\end{array}\right]
, \, 
 M_{\yy} = \left[\begin{array}{cccc}
0&0&0&0\\
1&0&0&0\\
0&-1/g&0&0\\
0&0&1&0
\end{array}\right]
, \, 
X_{\xx} = X_{\yy} = 
\left[\begin{array}{cccc}
0 & 0 & 0 & I/\ell
\end{array}\right]^\mathrm{T}$.
}:
\begin{equation}
\label{eq:zeta_x}
\begin{array}{l}
{\zeta}_{\zz} = {x}_{\zz},
\ \
{\zeta}_{\psi} = {x}_{\psi},
\\ \noalign{\smallskip}
{\zeta}_{\xx} = {x}_{\xx} +
\left[\begin{array}{cccc}
0 & 0 & 1 & {\pT}
\end{array}\right]^{\mathrm{T}}{q_{\xx}}({x}),
\\ \noalign{\smallskip}
{\zeta}_{\yy} = {x}_{\yy} +
\left[\begin{array}{cccc}
0 & 0 & 1 & {\pT}
\end{array}\right]^{\mathrm{T}}{q_{\yy}}({x}),
\end{array}
\end{equation}
which we apply to the model (\ref{eq:EcEdo_syst}) which then becomes:
\begin{equation}
\label{eq:Rep_edo_lin-pert}
\begin{array}{l}
\derT{{\zeta}_{\zz}}{t} = {A}_{\zz}\,{\zeta}_{\zz},
\ \
\derT{{\zeta}_{\psi}}{t} = {A}_{\psi}\,{\zeta}_{\psi},
\\ \noalign{\smallskip}
\derT{{\zeta}_{\xx}}{t} = {A}_{\xx}\,{\zeta}_{\xx} +
{B}_{\xx}\left({u}_{\xx} + {q}_{*,\xx}({x})\right),
\ \
{y}_{\xx} = {C}_{\xx}\,{\zeta}_{\xx},
\\ \noalign{\smallskip}
\derT{{\zeta}_{\yy}}{t} = {A}_{\yy}\,{\zeta}_{\yy} +
{B}_{\yy}\left({u}_{\yy} + {q}_{*,\yy}({x})\right),
\ \
{y}_{\yy} = {C}_{\yy}\,{\zeta}_{\yy}.
\end{array}
\end{equation}
where ${q}_{*,\xx}({x})$ and ${q}_{*,\yy}({x})$ are nonlinear uncertainty signals defined as follows\footnote{
${q}_{*}=X\left(S + \sum_{i=1}^{3}M^{i}S\psT{i}\right)q$. 
}:

\begin{equation}
\label{eq:q*}
\begin{array}{c}
{q}_{*,\xx}({x}) \triangleq
({I}/{\ell})
\left(\ders{{q}_{\xx}(x)}{t}{2}+
{q}_{\theta}(x)
\right)
\\ \noalign{\smallskip}
{q}_{*,\yy}({x}) \triangleq
({I}/{\ell})
\left(
\ders{{q}_{\yy}(x)}{t}{2} +
{q}_{\phi}(x)
\right)
\end{array}
\end{equation}
This new linearization procedure is an alternative to the well known \emph{input-output exact linearization} \cite{isidori,khalil,slotine}. The advantage of this method is that the nonlinearity appears in the linear form, $\bar{u}_{i} + {q}_{*,i}({x})$, $i \in \{\xx,\,\yy\}$, instead of appearing in the affine form,
$\alpha_{i}(x) + \beta_{i}(x)\bar{u}_{i}$, $i \in \{\xx,\,\yy\}$. This allows an exact cancellation of the nonlinear uncertainty signals using: ${u}_{i} = \bar{u}_{i} - {q}_{*,i}({x})$, $i \in \{\xx,\,\yy\}$. 

Then, the linearised model for the drone results in:
\begin{equation}
\label{eq:Rep_edo_lin-pert2}
\begin{array}{l}
\derT{{\zeta}_{\zz}}{t} = {A}_{\zz}\,{\zeta}_{\zz},
\ \
\derT{{\zeta}_{\psi}}{t} = {A}_{\psi}\,{\zeta}_{\psi},
\\ \noalign{\smallskip}
\derT{{\zeta}_{\xx}}{t} = {A}_{\xx}\,{\zeta}_{\xx} +
{B}_{\xx}\bar{u}_{\xx},
\ \
{p}_{\xx} = {C}_{\xx}\,{\zeta}_{\xx},
\\ \noalign{\smallskip}
\derT{{\zeta}_{\yy}}{t} = {A}_{\yy}\,{\zeta}_{\yy} +
{B}_{\yy}\bar{u}_{\yy},
\ \
{p}_{\yy} = {C}_{\yy}\,{\zeta}_{\yy}.
\end{array}
\end{equation}
From (\ref{eq:Rep_edo_lin-pert2}) we observe that the altitude $\mathrm{z}$ and the yaw angle $\psi$ are stable since ${A}_{\zz}$ and ${A}_{\psi}$ are Hurwitz matrices. Hence the linearized model (\ref{eq:Rep_edo_lin-pert2}) allows to control the position of the quadrotor in the horizontal plane while maintaining a constant altitude and yaw angle.

\section{}
\label{Appendix:B}
Let us consider a system described by the state space description:
\begin{equation}
\label{eq:MNCI_StSpDesc}
\dot{\mathbf{x}} = A\mathbf{x} + B\mathbf{u},
\quad
\mathbf{y} = C\mathbf{x},
\end{equation}
where: $\mathbf{x} \in \Rset^{n}$, $\mathbf{u} \in \Rset^{m}$, $\mathbf{y} \in \Rset^{p}$,
with the initial condition: $\mathbf{x}(0) = 0$.

We assume that the pair $(A,\,B)$ is controllable, namely: $\rk{\mcal{C}_{(A,\,B)}} = n$.
\medskip

We are interested in solving the following problem:

\begin{problem}
\label{Pr:MNCI_1}
Given a finite time $T > 0$ and given a vector $\mathbf{x}_{_T} \in \Rset^{n}$, find a minimal norm control input such that: $\mathbf{x}(T) = \mathbf{x}_{_T}$.
\end{problem}

This problem can be reformulated as follows:

\begin{problem}
\label{Pr:MNCI_2}
\begin{equation}
\label{eq:MNCI_Crit}
\displaystyle \minimize_{\mathbf{u}}\ \ \ \ \overline{\mcal{J}}(\mathbf{u}) =
\int_{0}^{T}\|\mathbf{u}(t)\|^2\,\dif{t}.
\end{equation}
{\rm s.t.}
\begin{equation}
\label{eq:MNCI_Rest}
\int_{0}^{T}
\mcal{F}^{\mathrm{T}}(T-t)\mathbf{u}(t)\,\dif{t} = \mathbf{x}_{_T}.
\end{equation}
{\rm where} $\mcal{F}(t) =B^{\mathrm{T}}\Expm{A^{\mathrm{T}}t}$.
\end{problem}

This is a classical minimum norm problem of seeking the closest vector to the origin lying in a variety of finite codimension, in a Hilbert space, and it is solved with the help of the Projection Theorem.
Indeed, according to Theorem 2 of \cite{luenberger}, the control input $\mathbf{u}$ which solves Problem \ref{Pr:MNCI_2} has the form:
\begin{equation}
\label{eq:MNCI_Defu}
\mathbf{u}(t) = \mcal{F}(T-t)\,\bm{\beta},
\end{equation}
where $\bm{\beta}$ is a vector in $\Rset^{n}$ satisfying (\ref{eq:MNCI_Rest}), that is to say:
\begin{equation}
\label{eq:MNCI_Defalpha}
\eu{W}_{_T}\,\bm{\beta} = \mathbf{x}_{_T},
\end{equation}

\noindent
where:
\begin{equation}
\label{eq:MNCI_DefW}
\eu{W}_{_T} =
\int_{0}^{T}
\Expm{A(T-t)}BB^{\mathrm{T}}\Expm{A^{\mathrm{T}}(T-t)}\,\dif{t}.
\end{equation}
$\eu{W}_{_T}$ is a well known matrix in system theory, which is used to study the reachability properties of (\ref{eq:MNCI_StSpDesc}), having the property: $\Rss{A}{\im{B}}$ $=$ $\im{\eu{W}_{_T}}$.
Furthermore, the controllability of the pair $(A,\,B)$ is equivalent to the invertibility of $\eu{W}_{_T}$ (see for example Theorem 1.1 of \cite{Wh:85}).

\begin{lemma}
\label{Lm:OptJbarr}
The solution of Problem \ref{Pr:MNCI_2} is:
\begin{equation}
\label{eq:MNCI_uSol}
\mathbf{u}_{*}(t) =\mcal{F}(T-t)\,\eu{W}_{_T}^{-1}\,\mathbf{x}_{_T}, \qquad 0 < t < T,
\end{equation}
where the optimal value of (\ref{eq:MNCI_Crit}) is:
\begin{equation}
\label{eq:OptJbarr}
\overline{\mcal{J}}_{*} =\mathbf{x}_{_T}^{\mathrm{T}}\,\eu{W}_{_T}^{-1}\,\mathbf{x}_{_T}.
\end{equation}
Moreover, for a given $\mathbf{x}_{_T} \neq 0$:
\begin{equation}
\label{eq:OptJbarr_cotas}
\lambda_{_{Max}}^{-1}({\eu{W}_{_T}}) \leq 
\overline{\mcal{J}}_{*}\Big/\|\mathbf{x}_{_T} \|^2\leq
\lambda_{_{min}}^{-1}({\eu{W}_{_T}}).
\end{equation}

\end{lemma}

\begin{proof}
(\ref{eq:MNCI_uSol}) follows from (\ref{eq:MNCI_Defu}), (\ref{eq:MNCI_Defalpha})
and the controllability of the pair $(A,\,B)$.

Substituting (\ref{eq:MNCI_uSol}) into (\ref{eq:MNCI_Crit}), we get (recall (\ref{eq:MNCI_DefW}) and the definition of $\mcal{F}(t)$): $\overline{\mcal{J}}_{*}$ $=$ $\mathbf{x}_{_T}^{\mathrm{T}}\,\left(\eu{W}_{_T}^{-1}\right)^{*}\,\mathbf{x}_{_T}$. Noting that $\eu{W}_{_T}$ is a Positive Definite Matrix (see (\ref{eq:MNCI_DefW})), we get (\ref{eq:OptJbarr}).

Then (\ref{eq:OptJbarr_cotas}) follows directly from the Rayleigh inequality (see for example \cite{stewart}).

\end{proof}
\section{}
\label{Appendix:C}
In this appendix we present the optimum quantization method used for $f_Q(\mathbf{p})$ (see (\ref{eq:3.0.6})) in the simulations section\footnote{Other optimization methods are also possible.}. The quantized function $f_Q(\mathbf{p})$ is defined as:
\begin{equation}
\label{eq.C.1}
f_Q(\mathbf{p})=R_j^Q, \ \ \ \forall\ \|\mathbf{p}\|\in(d_{j-1},d_{j}],
\end{equation}
where $j=1,\cdots,Q$, with $Q$ being the number of quantization levels; $d_{j-1}<d_{j}$; $R_Q^Q=\max_j\{R_j\}$, $d_Q=0$, $R_1^Q=0$ and $d_0=+\infty$ while $\{d_j\}_{j=1}^{Q-1}$ and $\{R_j^Q\}_{j=2}^{Q-1}$ are variables to be optimized in order to minimize the following error:
\begin{equation}
\label{eq.C.2}
E_Q=\int_{0}^{\|\mathbf{o}\|}\left(f_Q\left(\frac{\nu\mathbf{o}}{\|\mathbf{o}\|}\right)-\mathbb{E}\left[R\left({\Gamma}\left(\frac{\nu\mathbf{o}}{\|\mathbf{o}\|}\right)\right)\right]\right)^2d\nu
\end{equation}
where $\mathbf{o}=\mathbf{s}$ if $\|\mathbf{s}\|>\|\mathbf{g}\|$ and $\mathbf{o}=\mathbf{g}$ otherwise. Those variables can be optimized using numerical methods like simulated annealing \cite{b2}.
\end{appendices}

\end{document}